%% file: paper.tex
\documentclass[]{bytedance_seed}

\usepackage[toc,page,header]{appendix}

\usepackage{minitoc}

\usepackage[utf8]{inputenc}
\usepackage{url}
\usepackage{nicefrac}
\usepackage{amsthm}
\usepackage{array}
\usepackage{algorithm}
\usepackage{algpseudocode}
\usepackage{float}
\usepackage{enumitem}
\usepackage{wrapfig}
\usepackage{tikz}
\usetikzlibrary{arrows.meta,positioning}

\usepackage{tabularx}
\usepackage{makecell}
\usepackage{amsmath}
\usepackage{amssymb}

\newcolumntype{Y}{>{\raggedright\arraybackslash}X}
\newcolumntype{L}[1]{>{\raggedright\arraybackslash}p{#1}}

\algrenewcommand\algorithmicrequire{\textbf{Input:}}
\algrenewcommand\algorithmicensure{\textbf{Output:}}

\input{math_commands}

\definecolor{proofbg}{HTML}{F7F9FC}
\definecolor{proofframe}{HTML}{D8E2F0}

\newcommand{\code}[1]{{\ttfamily\upshape #1}}
\newtheorem{proposition}{Proposition}
\newtheorem{corollary}{Corollary}
\newtheorem{remark}{Remark}
\tcolorboxenvironment{proof}{
  enhanced,
  breakable,
  colback=proofbg,
  colframe=proofframe,
  boxrule=0.35pt,
  arc=1pt,
  left=6pt,
  right=6pt,
  top=4pt,
  bottom=4pt,
  before skip=6pt,
  after skip=6pt
}

\newtcolorbox{deepmemorybox}{
  enhanced,
  breakable,
  colback=seedblue!4,
  colframe=seedblue!4,
  boxrule=0.35pt,
  arc=1pt,
  left=6pt,
  right=6pt,
  top=5pt,
  bottom=5pt,
  before skip=6pt,
  after skip=6pt
}

\title{Modular TTT: Rethinking Test-Time Training as Composable Modules}

\author[1,2,3]{Bohao Tang}
\author[3,\ddagger]{Zhen Qin}
\author[3]{Yuqi Pan}
\author[3]{Zheng Li}
\author[1,2,\dagger]{Pengfei Liu}
\author[1,\dagger]{Ya Zhang}

\affiliation[1]{Shanghai Jiao Tong University}
\affiliation[2]{Shanghai Innovation Institute}
\affiliation[3]{ByteDance Seed}

\contribution[\ddagger]{Project lead}
\contribution[\dagger]{Corresponding authors}

\abstract{
Test-time training (TTT) views sequence modeling as an online learning problem in which fast weights are updated by an internal learning rule. Despite the growing number of TTT variants, existing approaches typically hard-code each variant separately, which makes it difficult to design new TTT methods and to isolate the role of each component. To address this, we propose Modular TTT, a framework that represents the inner learner as a directed acyclic graph and exposes the fast-weight network, loss function, learning rate, weight decay, and normalization as explicit design dimensions. Modular TTT automatically composes primitive-level train-view forward, train-view backward, and causal query-view rules into the full graph-level TTT computation, including the fast-weight state transition. Using Modular TTT, we systematically ablate the components of TTT and find that small learning-rate initialization, weight decay, and a single-layer nonlinearity improve performance, while MSE and inner-product losses perform similarly. Deeper fast-weight networks and normalization tend to hurt performance because they induce excessively large activations, while residual connections and gating provide little measurable benefit. Guided by these findings, we train the best resulting variant as 410M- and 1.45B-parameter models on 100B tokens, and observe training loss and benchmark performance comparable to Gated DeltaNet.
}

\date{\today}

\checkdata[Code]{\url{https://github.com/ByteDance-Seed/Modular-TTT}}

\begin{document}
\maketitle


\input{sections/introduction}
\input{sections/relatedwork}

\input{sections/approach}
\input{sections/experiments}

\clearpage

\bibliographystyle{plainnat}
\bibliography{main}

\clearpage

\beginappendix
\setcounter{equation}{0}
\renewcommand{\theequation}{A.\arabic{equation}}

\input{sections/appendix}

\end{document}

%% file: math_commands.tex
\usepackage{amsmath,amsfonts,bm}

\def\eqref#1{equation~\ref{#1}}

\def\1{\bm{1}}

\def\rmA{{\mathbf{A}}}
\def\rmB{{\mathbf{B}}}

\def\rmD{{\mathbf{D}}}

\def\rmH{{\mathbf{H}}}
\def\rmI{{\mathbf{I}}}

\def\rmK{{\mathbf{K}}}

\def\rmR{{\mathbf{R}}}

\def\rmU{{\mathbf{U}}}
\def\rmV{{\mathbf{V}}}

\def\vd{{\bm{d}}}

\def\vk{{\bm{k}}}

\def\vv{{\bm{v}}}

\def\mW{{\bm{W}}}

\DeclareMathAlphabet{\mathsfit}{\encodingdefault}{\sfdefault}{m}{sl}
\SetMathAlphabet{\mathsfit}{bold}{\encodingdefault}{\sfdefault}{bx}{n}



%% file: sections/introduction.tex

\section{Introduction}
Sequence modeling is a fundamental problem across many domains, and a broad range of architectures have been developed for it, including attention \citep{Vaswani+2017}, recurrent neural networks (RNNs)~\citep{Hochreiter1997LSTM,Chung2014GRU}, structured state space models (SSMs) \citep{Gu2022S4,GuDao2024Mamba}, and linear attention \citep{Katharopoulos2020LinearAttention}. Test-time training (TTT) \citep{Sun2024TTT} offers a different perspective by casting sequence modeling as an online learning process, in which fast weights \citep{Schmidhuber1992FastWeights,Ba2016FastWeights} are updated by an inner learning rule as the model processes a sequence. In contrast to methods that restrict state updates to fixed recurrences or explicit caches, TTT treats state update as inner learning, expanding the design space of sequence models.

Despite the growing number of TTT variants, the design space of TTT remains poorly understood. Existing approaches typically hard-code each variant separately, which creates two practical difficulties. First, it makes it difficult to develop new TTT methods in a systematic way, since modifying a variant often changes several components at once. Second, it obscures the role of individual components. In practice, the fast-weight network, loss function, learning rate, weight decay, and normalization are often changed together, making it difficult to determine which design choices are responsible for the observed behavior. As a result, while many TTT variants have been proposed \citep{Sun2024TTT,Zhang2025TTTRight,Behrouz2025Titans,Tandon2025E2ETTT}, their common structure and the contribution of each module remain unclear.

In this work, we present Modular TTT, a framework that represents the inner learner as a directed acyclic graph and factorizes TTT into explicit, controllable design dimensions. In Modular TTT, the fast-weight network, loss function, learning rate, weight decay, and normalization are treated as modular components that can be specified and varied independently. Given a learner DAG built from registered primitives, Modular TTT automatically composes their train-view forward, train-view backward, and causal query-view rules into the full graph-level TTT computation, including the fast-weight state transition. Automatic differentiation provides the local backward signals of the train-view loss; the causal query readout and fast-weight state transition are defined by the registered primitive rules. This removes the need to hand-derive a new global update rule for every TTT variant, and turns TTT design from a collection of custom implementations into a modular space that can be explored and analyzed systematically. Its optimized implementation also achieves multi-fold training-throughput improvements over the official implementations of the same TTT topologies. {Figure~\ref{fig:teaser} gives an overview of the modular formulation, speedup, and ablation results.}

\begin{figure*}[t]
\centering
\includegraphics[width=0.9\linewidth]{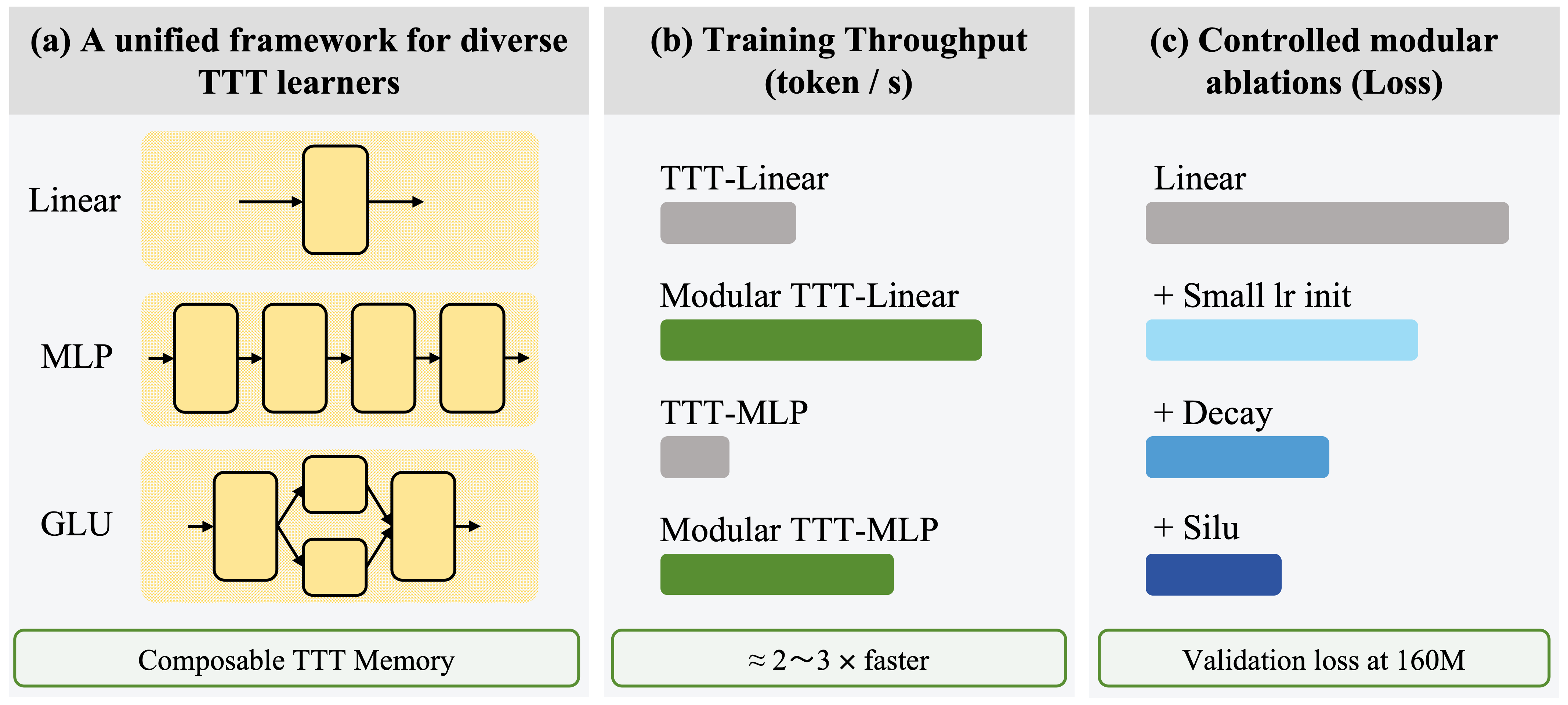}
\caption{\textbf{(a)} Linear, MLP, and gated learners {viewed as composable TTT memory forms}. \textbf{(b)} Training throughput vs.\ official TTT at 160M scale. \textbf{(c)} Ablations over key design choices: small lr init, decay, and SiLU each reduce validation loss.}
\label{fig:teaser}
\end{figure*}

\begin{figure}[t]
\centering
\includegraphics[width=1.0\linewidth]{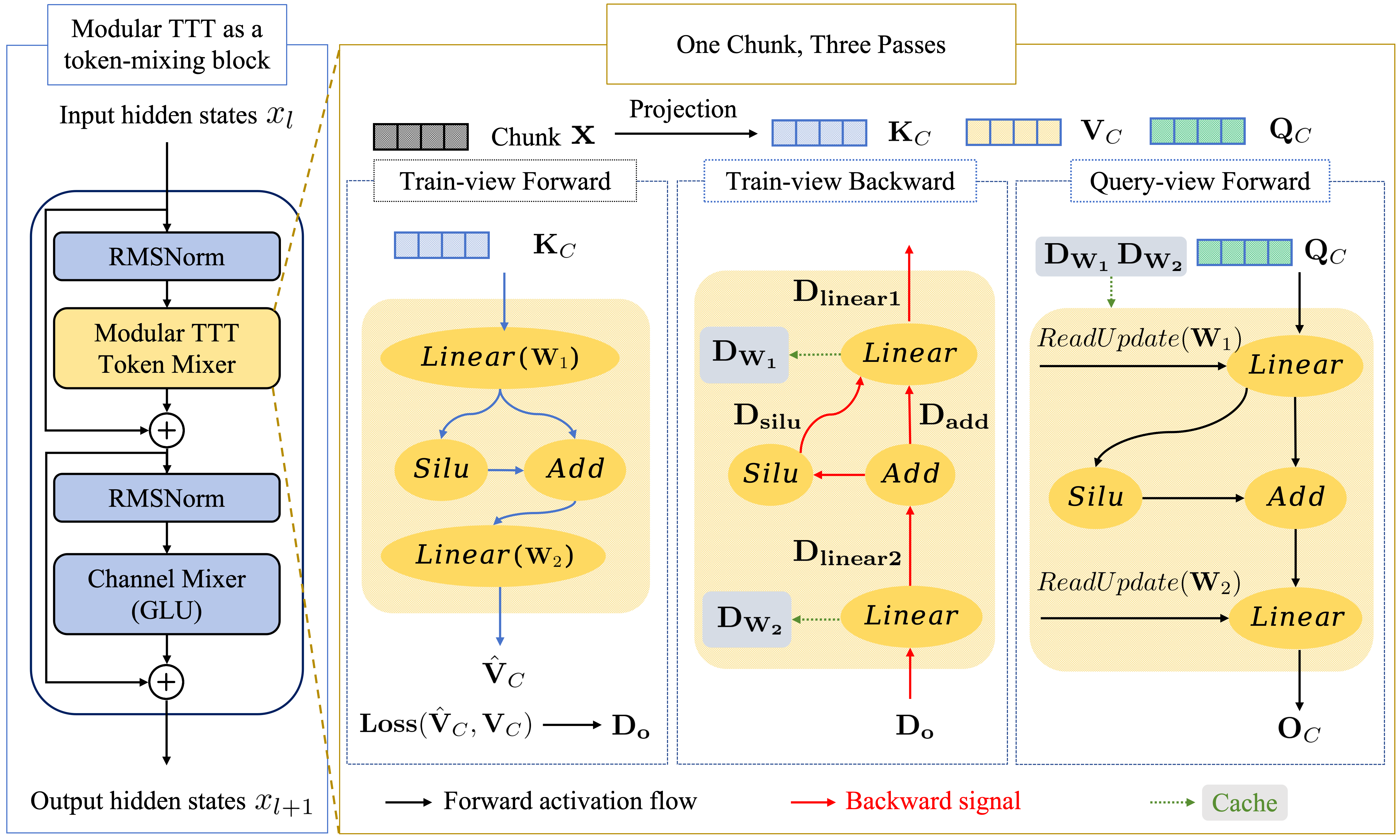}
\caption{Modular TTT: graph-structured memory with a shared {three-pass computation}. \textbf{Left:} the Modular TTT layer within a standard pre-norm backbone. \textbf{Right:} one chunk processed by train-view forward, train-view backward, and query-view forward over the graph memory. }
\label{fig:flash_ttt_overview}
\end{figure}

Building on Modular TTT, we conduct systematic ablations over the key components of TTT and obtain a clearer empirical picture of its design space. We find that small learning-rate initialization, weight decay, and a single-layer nonlinearity provide consistent gains, while MSE and inner-product losses perform similarly. Deeper fast-weight networks and normalization tend to degrade performance, likely due to excessively large activations, and residual connections and gating yield little measurable benefit. Guided by these findings, we train the selected variant at 410M/1.45B scale for 100B tokens and observe loss and benchmark performance comparable to Gated DeltaNet (GDN) \citep{Yang2025GDN}.

To summarize, our contributions are threefold. First, we introduce Modular TTT, a unified framework that represents the TTT inner learner as a directed acyclic graph and treats the fast-weight network, loss function, learning rate, weight decay, and normalization as modular design dimensions. Second, we use this framework to systematically analyze the role of major TTT components, identifying which design choices provide consistent gains and which offer little benefit under the evaluated settings. Third, guided by this analysis, we train the best resulting variant at 410M and 1.45B scale, and show that it achieves training loss and benchmark performance comparable to GDN.

%% file: sections/relatedwork.tex

\section{Related Work}

Existing work on efficient sequence modeling can be broadly grouped into four families: recurrent models, linear attention, state space models, and test-time training. All of these approaches aim to preserve long-context modeling ability while reducing the computation and memory costs that grow with sequence length, although they differ substantially in their implementations.

\textbf{RNN} Recurrent models compress historical context into a fixed-size hidden state, and enhance expressiveness through mechanisms such as gating and nonlinear transformations \citep{Hochreiter1997LSTM,Chung2014GRU}. The core idea is to accumulate contextual information through explicit state transition rules in a constant-size state, so that the cost of each inference step is decoupled from the history length. Recent work shows that removing nonlinearities from RNNs can improve efficiency while still achieving competitive performance \citep{Martin2018ParallelLRNN,Orvieto2023LRU,Qin2023HGRN,Peng2023RWKV}.

\textbf{Linear Attention} Linear attention rewrites attention as inner products of feature maps and exploits the associativity of matrix multiplication to aggregate key-value pairs first, thereby avoiding the explicit construction of a quadratic attention matrix \citep{Katharopoulos2020LinearAttention}. In the causal setting, reformulating token-wise recurrences into chunkwise recurrent computation \citep{Yang2024GLA} enables more efficient GPU utilization while preserving linear complexity with respect to sequence length. Subsequent work further extends this direction along several axes, including feature maps \citep{Choromanski2021Performer}, decay mechanisms \citep{Sun2023RetNet,qin2024transnormerllm,Yang2024GLA,Qin2024HGRN2,Yang2025DeltaNet,Yang2025GDN,Peng2025RWKV7} and normalization schemes.

\textbf{State Space Models} State space models start from continuous- or discrete-time state equations and model long-range dependencies through structured state transitions. Representative works such as S4 \citep{Gu2022S4} show that, with specialized initialization and structured parameterization, state space models can achieve strong empirical performance. More recent methods \citep{GuDao2024Mamba,DaoGu2024SSM} further emphasize content-dependent decay and hardware-friendly parallel scan algorithms, making this line a major direction for efficient sequence modeling.

\textbf{Test-Time Training} Unlike the approaches above, test-time training \citep{Krause2018Dynamic,Sun2024TTT} formulates sequence modeling as an online learning process, in which the hidden state is no longer a fixed-form vector state, but the fast weights \citep{Schmidhuber1992FastWeights,Ba2016FastWeights,Schlag2021FWP,Irie2021RecurrentFWP,Irie2022Dual} of a learnable model, and the corresponding state update is given by one or more optimization steps on a self-supervised objective. With suitable choices of the online learning loss and other components, the TTT perspective can instantiate a variety of Linear Attention, Linear RNN, and SSM variants. Recent work from this perspective has introduced instances such as TTT-Linear and TTT-MLP \citep{Sun2024TTT}, LaCT \citep{Zhang2025TTTRight}, and Titans \citep{Behrouz2025Titans}, demonstrating the potential of internal learners for long-context modeling. Subsequent work has further explored neural memory optimization \citep{Behrouz2025Atlas}, broader connections between TTT, retention, and online optimization \citep{Behrouz2025Miras,Behrouz2025Nested}, and extensions to vision \citep{Han2025ViT3}. However, naive TTT implementations typically rely on token-wise recurrent updates and therefore cannot fully exploit GPU parallelism. As in linear attention and SSMs, existing work often adopts chunkwise TTT implementations \citep{Zhang2025TTTRight,Li2025TNT} to achieve a better balance between computational efficiency and empirical performance.

%% file: sections/approach.tex

\section{Method}
\subsection{Analyzing the Computation of TTT}

Previous TTT methods typically hard-code both the forward and backward computations. However, we observe that this process can be automated. Consider an $m$-layer TTT module with nonlinear activation functions $f^{(i)}$, where $i\in\{1,\ldots,m\}$ indexes layers inside the TTT inner learner. Given an input $\mathbf K \in \mathbb R^{n \times d}$, where $n$ denotes the sequence length and $d$ denotes the feature dimension, we first perform a forward pass, which we refer to as the \emph{train-view forward}:

\begin{equation}
\begin{aligned}
\mathbf K^{(0)} &= \mathbf K, \\
d^{(0)} &= d, \\
\mathbf U^{(i)} &= \mathbf K^{(i-1)} \mathbf W^{(i)}, \\
\mathbf W^{(i)} &\in \mathbb R^{d^{(i-1)} \times d^{(i)}}, \\
\mathbf K^{(i)} &= f^{(i)}(\mathbf U^{(i)}) \in \mathbb R^{n \times d^{(i)}}, \\
\hat{\mathbf V} &= \mathbf K^{(m)}.
\end{aligned}
\end{equation}

Given a target $\mathbf V \in \mathbb R^{n \times d^{(m)}}$ and a loss function $\mathcal L$, we then perform a backward pass, which we refer to as the \emph{train-view backward}:
\begin{equation}
\begin{aligned}
l &= \mathcal L(\hat{\mathbf V}, \mathbf V), \\
\mathbf d\mathbf K^{(m)} &= \frac{\partial l}{\partial \hat{\mathbf V}}, \\
\mathbf d\mathbf U^{(i)}
&= \mathbf d\mathbf K^{(i)}
   \odot \left[f^{(i)}(\mathbf U^{(i)})\right]', \\
\mathbf d\mathbf K^{(i-1)}
&= \mathbf d\mathbf U^{(i)} [\mathbf W^{(i)}]^\top, \\
\mathbf d\mathbf W^{(i)}
&= [\mathbf K^{(i-1)}]^\top \mathbf d\mathbf U^{(i)}.
\end{aligned}
\end{equation}

Finally, we perform another forward pass using a query $\mathbf Q \in \mathbb R^{n \times d}$, which we refer to as the \emph{query-view forward}, where $\mathrm{Tril}(\mathbf X)$ denotes the lower-triangular part of $\mathbf X$ (i.e., $[\mathrm{Tril}(\mathbf X)]_{ij} = X_{ij}$ for $i \ge j$ and $0$ otherwise):
\begin{equation}
\begin{aligned}
\mathbf Q^{(0)} &= \mathbf Q, \\
\mathbf R^{(i)}
&= \mathbf Q^{(i-1)} \mathbf W^{(i)}
   - \mathrm{Tril}\!\left(
      \mathbf Q^{(i-1)}[\mathbf K^{(i-1)}]^\top
     \right)\mathbf d\mathbf U^{(i)}, \\
\mathbf W^{(i)} &= \mathbf W^{(i)} - \mathbf d\mathbf W^{(i)}, \\
\mathbf Q^{(i)} &= f^{(i)}(\mathbf R^{(i)}), \\
{\mathbf O} &= \mathbf Q^{(m)}.
\end{aligned}
\end{equation}

{This query-view form follows the causal dual update used in TTT~\citep{Sun2024TTT}.}
This example highlights several key observations:
\begin{itemize}[noitemsep, topsep=1pt]
    \item The train-view forward defines the computation performed by the TTT module.
    \item The train-view backward provides the activations required by the query-view forward.
    \item The query-view forward performs the actual forward computation of the TTT module.
    \item By defining the train-view and query-view operations for a single layer, we can derive the corresponding train-view and query-view operations for a multi-layer network.
\end{itemize}

\subsection{Modular TTT}

\begin{algorithm}[!t]
\caption{Modular TTT}
\label{alg:modular_ttt}
\small
\begin{algorithmic}[1]
\Require key $\mathbf K$, query $\mathbf Q$, target $\mathbf V$, topological order $\tau$, output node $o$, loss function $\mathcal L$, train-view operators $\{\phi_j^{\mathrm{train}}\}_{j \in \mathcal V}$, train-view backward operators $\{\phi_j^{\mathrm{train-bwd}}\}_{j \in \mathcal V}$, query-view operators $\{\phi_j^{\mathrm{query}}\}_{j \in \mathcal V}$, parameters $\{\theta_j\}_{j \in \mathcal V}$
\Ensure $\hat{\mathbf O}$

\Statex \textbf{Pass 1: train-view forward}
\State $\mathbf h_{\texttt{input}}^{\mathrm{train}} \gets \mathbf K$
\For{$j \in \tau$}
    \State $\mathbf h_j^{\mathrm{train}} \gets \phi_j^{\mathrm{train}}(\{\mathbf h_i^{\mathrm{train}}: i \in \mathrm{Par}(j)\};\,\theta_j)$
\EndFor
\State $\hat{\mathbf V} \gets \mathbf h_o^{\mathrm{train}}$

\Statex \textbf{Pass 2: train-view backward}
\State $l \gets \mathcal L(\hat{\mathbf V}, \mathbf V)$
\State Initialize $\bar{\mathbf h}_j^{\mathrm{train}} \gets \mathbf 0$ for all $j \in \mathcal V$, and set $\bar{\mathbf h}_o^{\mathrm{train}} \gets \partial l / \partial \hat{\mathbf V}$
\For{$j \in \mathrm{reverse}(\tau)$}
    \State $(\{\boldsymbol{\delta}_{i\leftarrow j}^{\mathrm{train}}\}_{i\in\mathrm{Par}(j)}, \Delta\theta_j) \gets \phi_j^{\mathrm{train-bwd}}(\bar{\mathbf h}_j^{\mathrm{train}}, \mathbf h_j^{\mathrm{train}}, \{\mathbf h_i^{\mathrm{train}}: i \in \mathrm{Par}(j)\}, \theta_j)$
    \For{$i \in \mathrm{Par}(j)$}
        \State $\bar{\mathbf h}_i^{\mathrm{train}} \gets \bar{\mathbf h}_i^{\mathrm{train}} + \boldsymbol{\delta}_{i\leftarrow j}^{\mathrm{train}}$
    \EndFor
\EndFor

\Statex \textbf{Pass 3: query-view forward}
\State $\mathbf h_{\texttt{input}}^{\mathrm{query}} \gets \mathbf Q$
\For{$j \in \tau$}
    \State $\mathbf h_j^{\mathrm{query}} \gets \phi_j^{\mathrm{query}}(\{\mathbf h_i^{\mathrm{query}}: i \in \mathrm{Par}(j)\}, {\{\mathbf h_i^{\mathrm{train}}: i \in \mathrm{Par}(j)\}, \bar{\mathbf h}_j^{\mathrm{train}}}, \theta_j, \Delta\theta_j)$
\EndFor
\State ${\mathbf O} \gets \mathbf h_o^{\mathrm{query}}$
\State \Return ${\mathbf O}$
\end{algorithmic}
\end{algorithm}

As in the example above, a general TTT module can be viewed as a computation graph. By defining the \texttt{train-view forward}, \texttt{train-view backward}, and \texttt{query-view forward} rules for each primitive in the graph, we can automatically compose the full graph-level TTT computation, without hand-deriving a new global update rule for every specific variant.

More concretely, we represent a TTT module as a directed acyclic graph $G=(\mathcal V,\mathcal E)$. Each node in the graph corresponds to a primitive operation, such as a linear map, an elementwise nonlinearity, a residual addition, or a normalization layer, while each edge represents a tensor dependency. Unlike prior work, which directly hard-codes a particular TTT variant, Modular TTT decomposes the inner learner into a set of local operations and assigns each of them explicit train-view and query-view semantics.

Let the primitive at node $j$ be denoted by $\phi_j$, with inputs from its parent set $\mathrm{pa}(j)$, parameters denoted by $\theta_j$ and $\mathbf h_j^{\text{train}}$ denotes the input and output of a node. In the train-view, we execute the forward pass of each node in topological order:
\begin{equation}
\begin{aligned}
\mathbf h_j^{\text{train}}
&=
\phi_j^{\text{train}}\big(
\{\mathbf h_i^{\text{train}} : i \in \mathrm{pa}(j)\};
\theta_j
\big),
\end{aligned}
\end{equation}
which produces the train-view output $\hat{\mathbf V}$, and we compute the inner learning loss $l = \mathcal L(\hat{\mathbf V}, \mathbf V)$. Next, we execute the train-view backward in reverse topological order to obtain intermediate gradients, local parameter updates, and backward signals to parent nodes. Denoting $\bar{\mathbf h}_j^{\text{train}} = \partial l / \partial \mathbf h_j^{\text{train}}$, the backward computation at node $j$ can be written as
\begin{equation}
\left(
\{\bar{\mathbf h}_i^{\text{train}}\}_{i \in \mathrm{pa}(j)},
\Delta \theta_j
\right)
=
\phi_j^{\text{train-bwd}}
\big(
\bar{\mathbf h}_j^{\text{train}},
\mathbf h_j^{\text{train}},
\{\mathbf h_i^{\text{train}}\}_{i \in \mathrm{pa}(j)},
\theta_j
\big),
\end{equation}

where $\Delta \theta_j$ denotes the parameter update induced by the inner learning rule. {The train-view activations and gradients, represented by $\{\mathbf h_i^{\text{train}}\}_{i \in \mathrm{pa}(j)}$ and $\bar{\mathbf h}_j^{\text{train}}$, are retained for the subsequent query-view computation.} The multi-layer MLP example above is precisely a special case of this general procedure on a computation graph.

After obtaining the train-view activations, backpropagated gradients, and parameter updates, we perform the query-view forward on a query input $\mathbf Q$. Unlike an ordinary forward pass, the query-view depends not only on the current input, but also on the {train-view activations, gradients,} and parameter updates produced by the train-view. For each node, we define the query-view computation as
\begin{equation}
\mathbf h_j^{\text{query}}
=
\phi_j^{\text{query}}
\big(
\{\mathbf h_i^{\text{query}} : i \in \mathrm{pa}(j)\},
{\{\mathbf h_i^{\text{train}} : i \in \mathrm{pa}(j)\},
\bar{\mathbf h}_j^{\text{train}}},
\theta_j,
\Delta \theta_j
\big).
\end{equation}

Automatic differentiation operates on an already specified computation graph. Applied to the train-view loss, it produces $\mathbf d\hat{\mathbf V}$ and the local backward signals propagated through the learner graph. The causal query-view readout and the fast-weight state transition require separate primitive-level specifications. Modular TTT registers these rules and composes them over the learner DAG. Executing the query-view forward in topological order then produces the final output $\hat{\mathbf O}$ of the TTT module.

This modular formulation provides two direct benefits. First, it significantly reduces the cost of constructing new TTT variants. Rather than manually deriving custom forward and backward rules for every new architecture, we only need to reorganize or replace local primitives in the graph. Second, it enables systematic ablations. Since different design factors correspond to different nodes or subgraphs, we can vary them in a more controlled way and directly analyze how each factor affects training loss and downstream performance. Therefore, Modular TTT should not be viewed as yet another specific TTT variant, but rather as a unified framework for expressing, implementing, and analyzing the design space of TTT. The full algorithmic procedure is given in Algorithm~\ref{alg:modular_ttt}. As summarized in Table~\ref{tab:modular_ttt_all}, we define a set of primitives in Modular TTT and specify for each primitive the corresponding $\phi^{\mathrm{train}}$, $\phi^{\mathrm{train\text{-}bwd}}$, and $\phi^{\mathrm{query}}$. Similar to TTT and linear attention, Modular TTT also incorporates learning-rate and decay mechanisms, and supports multiple choices of loss functions. Additional method details are provided in Appendix~\ref{app:framework_runtime}.

\begin{table*}[t]
\centering
\footnotesize
\setlength{\tabcolsep}{5pt}
\renewcommand{\arraystretch}{1.0}
\setlength{\aboverulesep}{0.1ex}
\setlength{\belowrulesep}{0.1ex}
\caption{Primitive operators and loss functions used in Modular TTT. Here, $\boldsymbol{\eta}$ denotes the learning rate and $\mathbf M$ denotes the decay matrix. For the \textsc{Gate} operator, $\mathrm{sum}(\cdot)$ and $\mathrm{cumsum}(\cdot)$ are taken along the sequence dimension. Here, $s$ denotes a normalization constant. When $s = n$, we average the loss over samples; otherwise, we use the unnormalized loss.
}
\label{tab:modular_ttt_all}

\textbf{Weight Operators}
\vspace{0.3em}

\begin{tabularx}{\textwidth}{@{}l
>{\hsize=0.60\hsize\raggedright\arraybackslash}X
>{\hsize=0.95\hsize\raggedright\arraybackslash}X
>{\hsize=1.35\hsize\raggedright\arraybackslash}X
@{}}
\toprule
Primitive & $\phi^{\mathrm{train}}$ & $\phi^{\mathrm{train\text{-}bwd}}$ & $\phi^{\mathrm{query}}$ \\
\midrule

Linear
& \makecell[l]{$\hat{\mathbf V} = \mathbf K \mathbf W$}
& \makecell[l]{$\mathbf d\mathbf K = \mathbf d\hat{\mathbf V}\mathbf W^\top$ \\
$\hat{\mathbf K} = \mathbf K \odot \boldsymbol{\eta},\ \mathbf d\mathbf W = \hat{\mathbf K}^\top \mathbf d\hat{\mathbf V}$}
& \makecell[l]{$\mathbf O = \mathbf Q \mathbf W - \mathrm{Tril}\!\left(\mathbf Q \hat{\mathbf K}^\top \odot \mathbf M\right)\mathbf d\hat{\mathbf V}$ \\
$\mathbf W = \mathbf W - \mathbf d\mathbf W$}
\\
\cmidrule(lr){2-2}\cmidrule(lr){3-3}\cmidrule(lr){4-4}

Gate
& \makecell[l]{$\hat{\mathbf V} = \mathbf K\,\mathrm{diag}(\mathbf W)$}
& \makecell[l]{$\mathbf d\mathbf K = \mathbf d\hat{\mathbf V}\,\mathrm{diag}(\mathbf W)$ \\
$\hat{\mathbf K} = \mathbf K \odot \boldsymbol{\eta}$ \\
$\mathbf d\mathbf W = \mathrm{sum}(\hat{\mathbf K} \odot \mathbf d\hat{\mathbf V})$}
& \makecell[l]{$\mathbf O = \mathbf Q\,\mathrm{diag}(\mathbf W) - \mathbf Q \odot \mathrm{cumsum}(\hat{\mathbf K} \odot \mathbf d\hat{\mathbf V})$ \\
$\mathbf W = \mathbf W - \mathbf d\mathbf W$}
\\

\bottomrule
\end{tabularx}

\vspace{0.9em}
\textbf{Non-weight Operators and Loss Functions}
\vspace{0.3em}

\begin{tabularx}{\textwidth}{l Y Y}
\toprule
Primitive & $\phi^{\mathrm{train}} / \phi^{\mathrm{query}}$ & $\phi^{\mathrm{train\text{-}bwd}}$ \\
\midrule

\multicolumn{3}{c}{\textit{Non-weight Operators}} \\

Norm
& \makecell[l]{$\hat{\mathbf v} = \mathbf k / \sigma$ \\
$\sigma = \sqrt{\left(\sum_{i=1}^{d} k_i^2\right)/d + \epsilon}$}
& \makecell[l]{$c = \left(\sum_{i=1}^{d} \hat v_i\, d\hat v_i\right)/d$ \\
$\mathbf d\mathbf k = (\mathbf d\hat{\mathbf v} - c\hat{\mathbf v})/\sigma$}
\\
\cmidrule(lr){2-2}\cmidrule(lr){3-3}

Act
& \makecell[l]{$\hat{\mathbf V} = f(\mathbf K)$}
& \makecell[l]{$\mathbf d\mathbf K = \mathbf d\hat{\mathbf V} \odot f'(\mathbf K)$}
\\
\cmidrule(lr){2-2}\cmidrule(lr){3-3}

Add
& \makecell[l]{$\hat{\mathbf V} = \mathbf K_1 + \mathbf K_2$}
& \makecell[l]{$\mathbf d\mathbf K_1 = \mathbf d\hat{\mathbf V},\ \mathbf d\mathbf K_2 = \mathbf d\hat{\mathbf V}$}
\\
\cmidrule(lr){2-2}\cmidrule(lr){3-3}

Mul
& \makecell[l]{$\hat{\mathbf V} = \mathbf K_1 \odot \mathbf K_2$}
& \makecell[l]{$\mathbf d\mathbf K_1 = \mathbf d\hat{\mathbf V} \odot \mathbf K_2,\ \mathbf d\mathbf K_2 = \mathbf d\hat{\mathbf V} \odot \mathbf K_1$}
\\

\midrule
\multicolumn{3}{c}{\textit{Loss Functions}} \\

Inner Product
& \makecell[l]{$l = -\left(\sum_{i=1}^{n} \hat{\mathbf v}_i^\top \mathbf v_i\right)/s$}
& \makecell[l]{$\mathbf d\hat{\mathbf V} = -\mathbf V/s$}
\\
\cmidrule(lr){2-2}\cmidrule(lr){3-3}

MSE
& \makecell[l]{$l = \left(\sum_{i=1}^{n} \lVert \hat{\mathbf v}_i - \mathbf v_i \rVert_2^2\right)/(2s)$}
& \makecell[l]{$\mathbf d\hat{\mathbf V} = (\hat{\mathbf V} - \mathbf V)/s$}
\\
\cmidrule(lr){2-2}\cmidrule(lr){3-3}

L1
& \makecell[l]{$l = \left(\sum_{i=1}^{n} \lVert \hat{\mathbf v}_i - \mathbf v_i \rVert_1\right)/s$}
& \makecell[l]{$\mathbf d\hat{\mathbf V} = \mathrm{sign}(\hat{\mathbf V} - \mathbf V)/s$}
\\
\cmidrule(lr){2-2}\cmidrule(lr){3-3}

RMSE
& \makecell[l]{$l = \sqrt{\left(\sum_{i=1}^{n} \lVert \hat{\mathbf v}_i - \mathbf v_i \rVert_2^2\right)/s}$}
& \makecell[l]{$\mathbf d\hat{\mathbf V} = (\hat{\mathbf V} - \mathbf V)\big/\sqrt{s \lVert \hat{\mathbf V} - \mathbf V \rVert_F^2}$}
\\

\bottomrule
\end{tabularx}
\end{table*}

%% file: sections/experiments.tex

\section{Experiments}
We implement Modular TTT in PyTorch, conduct language modeling experiments within the Flame framework \citep{yang2025flame}, and evaluate zero-shot downstream performance using \texttt{lm-eval-harness}~\citep{eval-harness}. We use Modular TTT to isolate the effects of existing and proposed TTT design choices under matched backbone, data, training budget, implementation, and evaluation settings. Our experiments consist of two stages. The first stage systematically ablates the major components of TTT, including the loss function, learning-rate initialization, decay, and nonlinearity. All experiments in this stage are conducted at 160M and 410M scale, with models trained on a large-scale English pretraining corpus for 10B tokens.

Based on the configurations identified in this ablation stage, we then perform a larger-scale study. In this stage, we train representative shortlisted variants on the same pretraining corpus for 100B tokens and compare them against representative baselines, including LLaMA, GDN \citep{Yang2025GDN}, and LaCT \citep{Zhang2025TTTRight} where available. We use the official TTT implementation \citep{Sun2024TTT} for {throughput and reproduction comparisons}. We use the OpenAI GPT-2 BPE tokenizer throughout. The sequence length is set to 2K in the ablation stage and 4K in the large-scale stage. For TTT variants, we use a chunk size of 256 to balance efficiency and empirical performance. Full architecture, training, and evaluation details are provided in Appendix~\ref{app:experiment_details}.
\subsection{Ablation}
\paragraph{Loss Ablation}

We first ablate the choice of loss function. Table~\ref{tab:loss_choice} reveals a clear two-level split: \textbf{MSE and inner product are the only competitive losses in the current TTT setting, while L1 and RMSE consistently underperform.} Across five paired runs at each scale, we observe only a marginal difference between MSE and inner-product loss (Appendix Table~\ref{tab:five_seed_robustness}).
This behavior can be partly explained by the fact that the loss only enters the fast update through the gradients
\begin{equation}
\label{eq:loss_gradient_update}
\begin{aligned}
\Delta \mathbf W
&= \hat{\mathbf K}^{\top}\mathbf d\hat{\mathbf V}, \\
\mathbf d\hat{\mathbf V}
&= \frac{\partial \mathcal L(\hat{\mathbf V},\mathbf V)}
        {\partial \hat{\mathbf V}} .
\end{aligned}
\end{equation}

{MSE preserves the residual magnitude in $\mathbf d\hat{\mathbf V}$, while inner product directly uses the target value as the write signal; both maintain an informative update scale. By contrast, L1 keeps only the sign of the residual, and RMSE normalizes the residual scale across the chunk, which weakens the memory write.} {A detailed analysis of the output gradients for these losses is provided in Appendix~\ref{app:loss_theory}.}

\begin{table*}[ht]
\centering
\begin{minipage}[t]{0.33\textwidth}
\centering
\captionof{table}{Experimental results for TTT loss functions. All methods use small-lr init and scalar decay. Only MSE and inner product are competitive choices, while L1 and RMSE perform substantially worse.}
\label{tab:loss_choice}
\scriptsize
\begin{tabular}{lcc}
\toprule
Loss & 160M & 410M \\
\midrule
L1    & 3.2665 & 2.9727 \\
Inner-product & 3.0383 & \textbf{2.7938} \\
MSE   & \textbf{3.0380} & 2.7949 \\
RMSE  & 3.0658 & 2.8041 \\
\bottomrule
\end{tabular}
\end{minipage}\hfill
\begin{minipage}[t]{0.62\textwidth}
\centering
\captionof{table}{Experimental results for TTT learning-rate initialization. Both MSE and inner-product losses are reported. Standard refers to initialization near 1; small-lr init sets $\eta_0 = 10^{-3}$. Small-lr init consistently achieves lower loss. Lower is better.
}
\label{tab:lr_init}
\footnotesize
\begin{tabular}{llcccc}
\toprule
\multirow{2}{*}{Loss} & \multirow{2}{*}{Decay} & \multicolumn{2}{c}{160M} & \multicolumn{2}{c}{410M} \\
\cmidrule(lr){3-4}\cmidrule(lr){5-6}
 &  & standard & small-lr init & standard & small-lr init \\
\midrule
\multirow{2}{*}{MSE}
 & none   & 3.6012 & 3.2005 & 3.4036 & 2.9343 \\
 & scalar & 3.3035 & \textbf{3.0380} & 3.0820 & 2.7949 \\
\midrule
\multirow{2}{*}{\shortstack{Inner-\\product}}
 & none   & 3.2407 & 3.2028 & 2.9572 & 2.9273 \\
 & scalar & 3.0538 & 3.0383 & 2.8054 & \textbf{2.7938} \\
\bottomrule
\end{tabular}
\end{minipage}
\end{table*}

\begin{table*}[ht]
\centering
\caption{Experimental results for TTT decay ablation. All methods use small-lr init and both MSE and inner-product losses are reported. Scalar decay provides the best efficiency--performance trade-off, while vector decay achieves the lowest loss at {lower throughput and higher memory cost}. Lower loss is better; higher Tgs is better. Tgs denotes training tokens per accelerator per second, and Mem is measured in GB.
}
\label{tab:decay_tradeoff}
\resizebox{\textwidth}{!}{%
\begin{tabular}{lrrrrrrrrrrrr}
\toprule
\multirow{3}{*}{Decay}
  & \multicolumn{6}{c}{MSE} & \multicolumn{6}{c}{Inner-product} \\
\cmidrule(lr){2-7}\cmidrule(lr){8-13}
  & \multicolumn{3}{c}{160M} & \multicolumn{3}{c}{410M}
  & \multicolumn{3}{c}{160M} & \multicolumn{3}{c}{410M} \\
\cmidrule(lr){2-4}\cmidrule(lr){5-7}\cmidrule(lr){8-10}\cmidrule(lr){11-13}
  & Loss & tgs & Mem & Loss & tgs & Mem
  & Loss & tgs & Mem & Loss & tgs & Mem \\
\midrule
none
  & 3.2005 & 109{,}399 & 27.54 & 2.9343 & 36{,}956 & 37.25
  & 3.2028 & 110{,}990 & 26.88 & 2.9273 & 37{,}668 & 35.56 \\
scalar
  & 3.0380 & 105{,}118 & 27.57 & 2.7949 & 34{,}978 & 37.30
  & 3.0383 & 105{,}878 & 26.91 & 2.7938 & 35{,}709 & 35.61 \\
vector
  & 3.0038 &  79{,}024 & 30.74 & 2.7821 & 27{,}074 & 40.81
  & 3.0048 &  80{,}156 & 29.54 & 2.7822 & 27{,}450 & 38.74 \\
\bottomrule
\end{tabular}}
\end{table*}

\paragraph{Learning-rate initialization}
\label{sec:sinlr}

We parameterize the learning rate following ~\citep{grazzi2024unlocking} as
$\eta_t = 2\,\mathrm{Sigmoid}(\beta_t + b)$, where $\beta_t = \mathbf x_t^\top \mathbf w \in \mathbb R$. 
Under the default initialization, $\beta_t$ is typically close to $0$ at the beginning of training, so $\eta_t \approx 2\,\mathrm{Sigmoid}(b)$.
We refer to the setting with $b=0$ and thus $\eta_t \approx 1$ as the \emph{default initialization}, and to the setting with
\[
\begin{aligned}
b
&= \log \frac{0.5 \times 10^{-3}}{1 - 0.5 \times 10^{-3}}, \\
\eta_t
&\approx 10^{-3},
\end{aligned}
\]
as \emph{small-lr init}. A similar initialization strategy is also adopted in the official TTT implementation \citep{Sun2024TTT} and in LaCT \citep{Zhang2025TTTRight}.

\textbf{TTT is highly sensitive to learning-rate initialization, and small-lr init consistently improves stability and final performance.} Under both MSE and inner product losses, small-lr init yields substantially lower loss, with a more pronounced improvement under MSE.

To understand this behavior, consider the parameter update under MSE:
\begin{equation}
\label{eq:mse_update}
\begin{aligned}
\mathbf W
&=
\mathbf W - \mathbf K^\top \mathrm{diag}(\boldsymbol{\eta})
(\mathbf K \mathbf W - \mathbf V)
\\
&=
\left(\mathbf I - \mathbf K^\top \mathrm{diag}(\boldsymbol{\eta}) \mathbf K\right)\mathbf W
+ \mathbf K^\top \mathrm{diag}(\boldsymbol{\eta}) \mathbf V.
\end{aligned}
\end{equation}
If the learning rate is too large, the matrix $
\mathbf I - \mathbf K^\top \mathrm{diag}(\boldsymbol{\eta}) \mathbf K$
may have eigenvalues with magnitude larger than $1$, which leads to instability. Therefore, small-lr init is beneficial for stable learning. {A detailed stability analysis of the update spectrum and chunk-scale effect is provided in Appendix~\ref{app:sinlr}.}

\paragraph{Decay Ablation}

Motivated by prior work on linear attention, we also study decay in the TTT setting, considering three variants: no decay, scalar decay, and vector decay. Consistent with earlier findings, Table~\ref{tab:decay_tradeoff} shows the same ordering under both losses. Without decay, past contributions to the fast weights accumulate without attenuation, which limits the model's ability to forget stale context. Scalar decay applies a global forgetting factor to past contributions at each step, multiplicatively contracting them and recovering most of the quality gain with almost no additional {computational overhead}. Vector decay extends this mechanism to feature-selective forgetting and achieves the best loss, but at the cost of roughly a 25\% drop in throughput and about 3\,GB higher peak memory at both model scales. \textbf{Scalar decay recovers most of the quality gain from vector decay at negligible {efficiency cost}; no decay is clearly weaker.} To balance efficiency and performance, we use scalar decay throughout the remaining discussion.

\paragraph{Role of Non-linearity}
To study the role of nonlinearity, we append {non-fast-weight} activations, including GELU and SiLU, after the linear module, and separately test Norm as a normalization operator. As shown in Table~\ref{tab:linear_nonlinearity}, \textbf{simple {non-fast-weight} activations consistently improve performance, while Norm is mixed; GELU/SiLU provide the best trade-off between quality and efficiency.} Linear-SiLU outperforms Linear in all five paired runs at both the 160M and 410M scales (Appendix Table~\ref{tab:five_seed_robustness}).

{For a linear node followed by an activation $\sigma$, the write to the preceding fast weight becomes}
\begin{equation}
\label{eq:activation_write}
{\Delta \mW = \rmK^\top\!\left(\rmD_o \odot \sigma'(\rmU)\right),
}
\end{equation}
where $\rmU=\rmK\mW$ and $\rmD_o$ is the gradient propagated from the activation output. Thus the nonlinearity increases expressivity by gating the write coordinatewise. Moreover, because the derivative is bounded, it does not induce excessively large activations.

Norm exhibits less stable results and can occasionally degrade performance. We attribute this instability in part to the fact that normalization can lead to larger activations. Note that the gradient of Norm is: 
\begin{equation}
\label{eq:rmsnorm_backward}
\begin{aligned}
\mathbf d\mathbf z
&=
\frac{1}{\sigma(\mathbf z)}
\left(
\mathbf d\mathbf y
-
\mathrm{mean}(\mathbf y\odot\mathbf d\mathbf y)\mathbf y
\right), \\
\mathbf y
&= \mathbf z/\sigma(\mathbf z), \\
\sigma(\mathbf z)
&= \sqrt{\|\mathbf z\|_2^2/d+\epsilon}.
\end{aligned}
\end{equation}
The $1/\sigma(\mathbf z)$ factor can amplify gradient when $\sigma(\mathbf z)$ is small. This helps explain why Norm is less robust than pointwise activations in Table~\ref{tab:linear_nonlinearity}. 

\begin{table*}[!t]
\centering
\caption{Experimental results for nonlinear operators. All methods use small-lr init and scalar decay, without mean scaling. GELU and SiLU consistently improve performance and provide the best trade-off between {throughput} and final loss.
}
\label{tab:linear_nonlinearity}
\resizebox{\textwidth}{!}{%
\begin{tabular}{lrrrrrrrrrrrr}
\toprule
\multirow{3}{*}{Method}
  & \multicolumn{6}{c}{MSE} & \multicolumn{6}{c}{Inner-product} \\
\cmidrule(lr){2-7}\cmidrule(lr){8-13}
  & \multicolumn{3}{c}{160M} & \multicolumn{3}{c}{410M}
  & \multicolumn{3}{c}{160M} & \multicolumn{3}{c}{410M} \\
\cmidrule(lr){2-4}\cmidrule(lr){5-7}\cmidrule(lr){8-10}\cmidrule(lr){11-13}
  & Loss & tgs & Mem & Loss & tgs & Mem
  & Loss & tgs & Mem & Loss & tgs & Mem \\
\midrule
Linear
  & 3.0380          & 105{,}118 & 27.57 & 2.7949 & 34{,}978 & 37.30
  & 3.0383          & 105{,}878 & 26.91 & 2.7938          & 35{,}709 & 35.61 \\
Linear + GELU
  & 3.0232          &  97{,}808 & 30.98 & 2.7844 & 32{,}114 & 42.83
  & 3.0225          &  98{,}405 & 30.97 & \textbf{2.7832} & 32{,}456 & 43.07 \\
Linear + SiLU
  & \textbf{3.0205} &  98{,}293 & 30.98 & 2.7849 & 32{,}129 & 42.83
  & 3.0242          &  98{,}858 & 30.97 & 2.7838          & 32{,}453 & 43.07 \\
Linear + Norm
  & 3.0300          &  93{,}226 & 32.15 & 2.8192 & 30{,}501 & 44.76
  & 3.0335          &  94{,}625 & 31.01 & 2.7899          & 31{,}139 & 43.11 \\
\bottomrule
\end{tabular}}
\end{table*}

\paragraph{Deep Memory ablation}

We next study whether replacing a single linear layer with an MLP brings additional gains. Under our controlled language-modeling setting and stabilization sweeps, Table~\ref{tab:multilayer_90m} gives a negative answer:

\textbf{Deeper graph learners do not surpass the shallow frontier, and shallow linear variants remain the strongest among the tested configurations.} This difficulty already appears in the two-layer linear case, and the same barrier remains after adding activations or normalization. Table~\ref{tab:multilayer_90m} keeps the main-text comparison focused on chain-structured deep learners. Across these variants, added depth changes the number of layers but does not change the qualitative conclusion: activation placement gives only limited gains, and norm-containing variants either diverge or remain behind the shallow frontier. Residual and gated variants are reported in Appendix~\ref{app:family_support}, where they are analyzed together with their stabilization sweeps. Appendix~\ref{app:deep_theory} provides the corresponding factor-coupling analysis for deeper fast learners.

\begin{table*}[t]
\centering
\footnotesize
\renewcommand{\arraystretch}{0.96}
\setlength{\tabcolsep}{3pt}
\caption{Deep TTT variants at 160M. {Residual/gated variants and full stabilized sweeps are in Appendix~\ref{app:family_support}.} Lower is better; $\times$ denotes divergence.}
\label{tab:multilayer_90m}
\begin{tabular*}{\textwidth}{@{\extracolsep{\fill}}lrlrlr@{}}
\toprule
\multicolumn{2}{c}{Shallow / Product}
&
\multicolumn{2}{c}{Activation-only}
&
\multicolumn{2}{c}{Norm-containing}
\\
\cmidrule(lr){1-2}
\cmidrule(lr){3-4}
\cmidrule(l){5-6}
Method & Loss & Method & Loss & Method & Loss \\
\midrule
Linear
& 3.0380
& Linear-Linear-SiLU
& 3.1223
& Linear-Linear-Norm
& $\times$
\\
Linear-SiLU
& \textbf{3.0205}
& Linear-SiLU-Linear
& 3.1240
& Linear-Norm-Linear
& 3.1594
\\
Linear-Norm
& 3.0300
& Linear-SiLU-Linear-SiLU
& 3.1156
& Linear-Norm-Linear-Norm
& $\times$
\\
Linear-Linear
& 3.1265
&
&
& Linear-SiLU-Linear-Norm
& 3.1144
\\
\bottomrule
\end{tabular*}
\end{table*}


\begin{deepmemorybox}

\textbf{Why Is Deep TTT Memory more difficult to optimize? }

Table~\ref{tab:multilayer_90m} gives a consistent negative result: increasing the depth of the fast-weight memory does not improve over the shallow frontier. We argue that this is not merely an optimization accident, but a structural consequence of how one-step TTT updates interact with factorized fast weights. Consider a TTT memory of the form
\begin{equation}
\mathbf Y = f(\mathbf X) = \mathbf X \mathbf W^{(1)} \mathbf W^{(2)}.
\end{equation}
We compare two parameterizations. First, treat the product $\mathbf W^{(1)} \mathbf W^{(2)} \triangleq \mathbf W$ as a single fast weight. 

In this case, the update is given by $\Delta = \mathbf X^\top \mathbf{dY}$ and $\mathbf W = \mathbf W - \Delta$.

Now consider the two-factor memory. Define
\begin{equation}
\mathbf X^{(1)} \triangleq \mathbf X \mathbf W^{(1)}, \qquad
\mathbf X^{(2)} = \mathbf Y \triangleq \mathbf X^{(1)} \mathbf W^{(2)}.
\end{equation}

The corresponding updates are
\begin{equation}
\begin{aligned}
\Delta^{(2)}
&= [\mathbf X^{(1)}]^\top \mathbf{dY} 
= [\mathbf X \mathbf W^{(1)}]^\top \mathbf{dY} 
= [\mathbf W^{(1)}]^\top \mathbf X^\top \mathbf{dY} 
= [\mathbf W^{(1)}]^\top \Delta, \\
\mathbf{dX}^{(1)}
&= \mathbf{dX}^{(2)} [\mathbf W^{(2)}]^\top 
= \mathbf{dY} [\mathbf W^{(2)}]^\top, \\
\Delta^{(1)}
&= \mathbf X^\top \mathbf{dX}^{(1)} 
= \mathbf X^\top \mathbf{dY} [\mathbf W^{(2)}]^\top 
= \Delta [\mathbf W^{(2)}]^\top .
\end{aligned}
\end{equation}

Therefore, the induced update on the effective fast weight becomes
\begin{equation}
\begin{aligned}
&(\mathbf W^{(1)} - \Delta [\mathbf W^{(2)}]^\top)
(\mathbf W^{(2)} - [\mathbf W^{(1)}]^\top \Delta) \\
&=
\mathbf W^{(1)}\mathbf W^{(2)}
- \Delta [\mathbf W^{(2)}]^\top \mathbf W^{(2)}
- \mathbf W^{(1)}[\mathbf W^{(1)}]^\top \Delta
+ \Delta [\mathbf W^{(2)}]^\top [\mathbf W^{(1)}]^\top \Delta .
\end{aligned}
\end{equation}

For any nonzero scalar $c$, we have $\mathbf W^{(1)}\mathbf W^{(2)} = (c\mathbf W^{(1)})(c^{-1}\mathbf W^{(2)})$.

Thus the represented memory function $f(\mathbf X)=\mathbf X\mathbf W$ remains unchanged.

However, under this rescaling, the two-layer update becomes
\begin{equation}
\Delta [\mathbf W^{(2)}]^\top [\mathbf W^{(1)}]^\top \Delta
- c^2\mathbf W^{(1)}\mathbf W^{(1)\top}\Delta
- c^{-2}\Delta\mathbf W^{(2)\top}\mathbf W^{(2)}.
\end{equation}

Therefore, two factorizations that represent the same effective fast weight can induce substantially different TTT update directions.

When \(c\) is large, the dominant terms are
\begin{equation}
\Delta [\mathbf W^{(2)}]^\top [\mathbf W^{(1)}]^\top \Delta
- c^2\mathbf W^{(1)}\mathbf W^{(1)\top}\Delta .
\end{equation}

When \(c\) is small, the dominant terms are
\begin{equation}
\Delta [\mathbf W^{(2)}]^\top [\mathbf W^{(1)}]^\top \Delta
- c^{-2}\Delta\mathbf W^{(2)\top}\mathbf W^{(2)}.
\end{equation}

Therefore, optimization is not only required to learn the correct effective fast weight \(\mathbf W\), but also to discover a factorization that induces favorable update dynamics. This additional degree of freedom makes deep TTT Memory substantially more difficult to optimize.

\end{deepmemorybox}

\subsection{Efficiency Benchmark}
\label{sec:throughput}
The modular formulation reduces the cost of constructing and comparing TTT variants. The following benchmarks separately evaluate the efficiency of primitive-level backward computation and the end-to-end training efficiency of fixed learner graphs.

\textbf{Primitive-level backward computation.} To isolate the cost of executing local backward rules, we compare the analytic backward operators for Linear and Norm with reference implementations based on \texttt{torch.autograd.grad} under identical inputs. The Linear and Norm input shapes are $(2,128,12,128,64)$ and $(2,128,12,128)$, respectively. We use 20 warm-up iterations and 100 measured iterations.

\begin{table*}[!t]
\centering
\caption{Microbenchmark of primitive-level backward computation. We report mean latency and maximum memory over 100 measured iterations after 20 warm-up iterations. Speedup is the autodiff-reference latency divided by the analytic-operator latency. Lower latency and memory are better.}
\label{tab:primitive_backward}
\small
\begin{tabular*}{\textwidth}{@{\extracolsep{\fill}}llccc@{}}
\toprule
Primitive & Implementation & Latency (ms) & Max mem. (MB) & Speedup \\
\midrule
Linear & Analytic & \textbf{0.3127} & 20.4 & $1.65\times$ \\
Linear & Autodiff ref. & 0.5161 & 19.6 & -- \\
Norm & Analytic & \textbf{0.3146} & \textbf{19.3} & $2.62\times$ \\
Norm & Autodiff ref. & 0.8229 & 31.3 & -- \\
\bottomrule
\end{tabular*}
\end{table*}

The analytic Linear and Norm operators yield $1.65\times$ and $2.62\times$ speedups over their autodiff references, respectively. Peak memory is comparable for Linear (20.4 versus 19.6 MB). For Norm, the analytic operator reduces peak memory from 31.3 to 19.3 MB. Under the evaluated software configuration, our inner-loop reference path based on \texttt{torch.autograd.grad} is incompatible with \texttt{torch.compile}. The custom analytic operators are included in the compiled graph.

\textbf{End-to-end throughput.} Table~\ref{tab:speed_vs_official} compares Modular TTT with the official TTT implementation \citep{Sun2024TTT} in two representative 160M norm-containing settings. Modular TTT achieves a $2.2\times$--$3.3\times$ throughput improvement over the official implementation.

\begin{table*}[!t]
\centering
\caption{{Throughput validation} in two representative 160M norm-containing settings on GPU. TTT-Linear and TTT-MLP follow the nomenclature of \citet{Sun2024TTT}. Modular TTT achieves about $2.2\times$--$3.3\times$ the training throughput of the official TTT implementation.}
\label{tab:speed_vs_official}
\small
\begin{tabular*}{\textwidth}{@{\extracolsep{\fill}}lccc@{}}
\toprule
Setting & Official TTT & Modular TTT & Speedup \\
\midrule
TTT-Linear (Linear + Norm) & 43{,}397.5 & 93{,}366.9 & ${\approx}2.2\times$ \\
TTT-MLP (Linear + GELU + Linear + Norm) & 21{,}336.9 & 71{,}101.1 & ${\approx}3.3\times$ \\
\bottomrule
\end{tabular*}
\end{table*}

\subsection{Scale Up}
\label{sec:large_scale}

Based on the preceding ablations, we select representative variants that balance validation loss, {throughput}, and simplicity. We keep MSE and inner product as the two competitive loss functions, use scalar decay as the efficiency-oriented decay choice rather than the absolute lowest-loss option, and compare a single linear learner with a linear learner followed by a SiLU activation.

Table~\ref{tab:large_scale_eval} reports downstream benchmarks at 410M and 1.45B, covering perplexity, multiple-choice, and containment-style tasks evaluated with lm-evaluation-harness \citep{eval-harness}. The shortlisted Modular TTT variants remain competitive with strong recurrent baselines on training loss and multiple-choice accuracy, while containment-style tasks remain more challenging. At 410M, the SiLU variant improves the multiple-choice average over the external baselines; at 1.45B, the inner-product linear variant gives the strongest Modular TTT result and is close to GDN on average multiple-choice accuracy. Appendix~\ref{app:external_repro} reports additional external-baseline checks, including the LaCT comparison.

\newcommand{\avgcell}[1]{\begingroup\setlength{\fboxsep}{1pt}\colorbox{seedblue!8}{\strut #1}\endgroup}
\newcommand{\metrichead}[2]{\shortstack{#1\\#2}}

\begin{table*}[!t]
\centering
\scriptsize
\setlength{\tabcolsep}{1.5pt}
\caption{Large-scale downstream evaluation at 410M and 1.45B. Arrows indicate the direction of improvement. External baselines: LLaMA \citep{Touvron2023Llama2}, GDN \citep{Yang2025GDN}, and LaCT \citep{Zhang2025TTTRight} where available.}
\label{tab:large_scale_eval}
\resizebox{\textwidth}{!}{%
\begin{tabular}{@{}lc|cccc|ccccccccc|cccc@{}}
\toprule
\multirow{2}{*}{Method} & \multirow{2}{*}{Params} & \multicolumn{4}{c|}{Perplexity} & \multicolumn{9}{c|}{Multiple-choice} & \multicolumn{4}{c}{Containment} \\
\cmidrule(lr){3-6}\cmidrule(lr){7-15}\cmidrule(lr){16-19}
 & & \avgcell{loss$\downarrow$} & \metrichead{Wiki.}{ppl$\downarrow$} & \metrichead{LMB.}{ppl$\downarrow$} & \avgcell{\metrichead{Avg}{ppl$\downarrow$}} & \metrichead{BoolQ}{acc$\uparrow$} & \metrichead{PIQA}{acc$\uparrow$} & \metrichead{Hella.}{acc-n$\uparrow$} & \metrichead{Wino.}{acc$\uparrow$} & \metrichead{ARC-e}{acc$\uparrow$} & \metrichead{ARC-c}{acc-n$\uparrow$} & \metrichead{OBQA}{acc-n$\uparrow$} & \metrichead{SIQA}{acc$\uparrow$} & \avgcell{\metrichead{Avg}{acc$\uparrow$}} & \metrichead{SWDE}{acc$\uparrow$} & \metrichead{SQuAD}{acc$\uparrow$} & \metrichead{FDA}{acc$\uparrow$} & \avgcell{\metrichead{Avg}{acc$\uparrow$}} \\
\midrule
\multicolumn{19}{@{}l}{\textit{External baselines, 410M}} \\
LLaMA & 410M & \avgcell{2.5479} & 20.63 & 25.64 & \avgcell{23.14} & 59.51 & 69.15 & 47.47 & 55.80 & 65.28 & 31.83 & 35.40 & 40.63 & \avgcell{50.63} & 58.96 & 40.42 & 62.43 & \avgcell{53.94} \\
GDN   & 410M & \avgcell{2.5559} & 21.00 & 22.62 & \avgcell{21.81} & 59.69 & 71.00 & 47.64 & 52.88 & 65.66 & 31.83 & 34.80 & 40.07 & \avgcell{50.45} & 22.68 & 31.80 & 16.33 & \avgcell{23.60} \\
LaCT  & 410M & \avgcell{2.5582} & 21.28 & 25.94 & \avgcell{23.61} & 60.86 & 70.13 & 47.10 & 52.49 & 65.03 & 30.03 & 35.80 & 40.02 & \avgcell{50.18} & 33.75 & 30.16 & 18.60 & \avgcell{27.50} \\
\midrule
\multicolumn{19}{@{}l}{\textit{Modular TTT MSE Loss, 410M}} \\
M      & 410M & \avgcell{2.5648} & 21.28 & 25.10 & \avgcell{23.19} & 53.76 & 69.04 & 47.92 & 52.33 & 67.21 & 31.74 & 35.40 & 41.10 & \avgcell{49.81} & 25.02 & 32.14 &  8.53 & \avgcell{21.90} \\
M-SiLU & 410M & \avgcell{2.5582} & 21.07 & 23.80 & \avgcell{22.43} & 61.44 & 70.40 & 48.10 & 52.88 & 65.66 & 33.87 & 35.00 & 40.53 & \avgcell{50.99} & 28.08 & 31.77 &  8.80 & \avgcell{22.88} \\
\midrule
\multicolumn{19}{@{}l}{\textit{Modular TTT Inner-Product Loss, 410M}} \\
M      & 410M & \avgcell{2.5621} & 21.38 & 24.81 & \avgcell{23.10} & 56.30 & 69.37 & 47.71 & 53.04 & 64.52 & 30.89 & 37.40 & 39.41 & \avgcell{49.83} & 26.28 & 33.91 &  9.62 & \avgcell{23.27} \\
M-SiLU & 410M & \avgcell{2.5569} & 20.98 & 23.22 & \avgcell{22.10} & 60.03 & 70.62 & 48.14 & 53.20 & 66.50 & 31.91 & 34.40 & 39.92 & \avgcell{50.59} & 25.38 & 31.84 & 12.89 & \avgcell{23.37} \\
\midrule
\multicolumn{19}{@{}l}{\textit{External baselines, 1.45B}} \\
LLaMA & 1.45B & \avgcell{2.3041} & 15.26 & 12.78 & \avgcell{14.02} & 62.91 & 73.83 & 59.37 & 57.22 & 73.48 & 38.31 & 39.80 & 40.43 & \avgcell{55.67} & 75.34 & 44.50 & 75.77 & \avgcell{65.20} \\
GDN   & 1.45B & \avgcell{2.3042} & 15.53 & 11.87 & \avgcell{13.70} & 60.70 & 74.37 & 59.16 & 60.54 & 73.61 & 39.68 & 41.00 & 41.30 & \avgcell{56.30} & 52.39 & 38.37 & 20.15 & \avgcell{36.97} \\
\midrule
\multicolumn{19}{@{}l}{\textit{Modular TTT MSE Loss, 1.45B}} \\
M      & 1.45B & \avgcell{2.3193} & 15.84 & 12.71 & \avgcell{14.28} & 60.37 & 74.43 & 58.96 & 58.41 & 72.94 & 40.61 & 41.80 & 41.71 & \avgcell{56.15} & 41.40 & 38.71 & 24.77 & \avgcell{34.96} \\
M-SiLU & 1.45B & \avgcell{2.3197} & 15.85 & 12.90 & \avgcell{14.37} & 59.91 & 72.91 & 58.92 & 58.64 & 73.23 & 40.02 & 40.60 & 41.15 & \avgcell{55.67} & 44.19 & 37.77 & 24.41 & \avgcell{35.46} \\
\midrule
\multicolumn{19}{@{}l}{\textit{Modular TTT Inner-Product Loss, 1.45B}} \\
M      & 1.45B & \avgcell{2.3150} & 15.92 & 12.19 & \avgcell{14.05} & 62.35 & 73.18 & 58.70 & 59.98 & 73.48 & 40.78 & 41.00 & 42.07 & \avgcell{56.44} & 43.65 & 37.90 & 17.42 & \avgcell{32.99} \\
M-SiLU & 1.45B & \avgcell{2.3156} & 15.84 & 12.31 & \avgcell{14.07} & 56.97 & 74.21 & 58.80 & 58.64 & 73.74 & 39.93 & 39.20 & 42.32 & \avgcell{55.48} & 42.30 & 38.00 & 17.79 & \avgcell{32.70} \\
\bottomrule
\end{tabular}
}
\end{table*}

\paragraph{{Per-token Loss.}}
Appendix~\ref{app:per_token_loss} further examines per-token loss along 8k sequences, showing that the selected shallow Modular TTT variants remain stable around and beyond the 4k training context.
\paragraph{{RULER.}}
Detailed RULER \citep{Hsieh2024RULER} evaluations are reported in Appendix~\ref{app:ruler_details}, with family-level and template-level results in Tables~\ref{tab:ruler_single_family_full}--\ref{tab:ruler_multi_family_full}.
The inner-product linear variant is the strongest Modular TTT configuration on \texttt{niah\_single} at 410M, while the MSE SiLU variant is strongest at 1.45B for contexts up to 4k. However, LLaMA remains substantially stronger, especially at 8k, indicating that precise long-context recall remains a limitation of the current shallow Modular TTT variants.

\section{Conclusion}

We presented Modular TTT, a framework that factorizes the TTT inner learner into explicit, independently controllable design dimensions. Modular TTT represents the fast-weight network as a DAG of registered primitives and automatically composes their train-view forward, train-view backward, and causal query-view rules into the full graph-level TTT computation, including the fast-weight state transition. This shared procedure reduces the cost of constructing and comparing TTT variants by avoiding topology-specific global update derivations for recombinations of existing primitives, and turns TTT design from a collection of custom implementations into a modular space that can be explored and analyzed systematically. Using this framework, we identified which components of TTT provide consistent gains and which do not, and trained the resulting best variant to performance comparable to GDN at 410M and 1.45B scales.

%% file: sections/appendix.tex

\makeatletter
\newcommand{\appendixcontents}{%
  \section*{Appendix Contents}%
  \label{app:roadmap}%
  \begingroup
  \setlength{\parskip}{3pt}%
  \def\l@appendixsection##1##2{\par\noindent ##1 \dotfill ##2\par}%
  \renewcommand\numberline[1]{\textbf{Appendix~##1}\quad}%
  \@starttoc{apc}%
  \endgroup
}
\newcommand{\appsection}[1]{%
  \section{#1}%
  \addcontentsline{apc}{appendixsection}{\protect\numberline{\thesection}#1}%
}
\makeatother

\appendixcontents

\appsection{Method details}
\label{app:framework_runtime}

\subsection{Design axes}
\label{app:design_axes}

\begin{table}[ht]
\centering
\small
\renewcommand{\arraystretch}{1.18}
\setlength{\tabcolsep}{5pt}
\caption{Configuration axes in Modular TTT. The table summarizes the graph, loss, learning-rate, decay, normalization, and scheduling choices controlled by our implementation and experiments. Primitive and loss definitions follow Table~\ref{tab:modular_ttt_all}.}
\label{tab:instantiations}
\begin{tabular}{%
  >{\raggedright\arraybackslash}p{0.26\linewidth}%
  >{\raggedright\arraybackslash}p{0.44\linewidth}%
  >{\raggedright\arraybackslash}p{0.18\linewidth}%
}
\toprule
Axis & Options & Component \\
\midrule
\multicolumn{3}{l}{\textit{Graph and loss choices}} \\[2pt]
Primitive
  & Linear,\ Gate,\ Norm,\ Act,\ Add,\ Mul
  & Node operator $\phi_j$ \\
Loss
  & Inner product,\ MSE,\ L1,\ RMSE
  & Loss $\mathcal L$ \\
\midrule
\multicolumn{3}{l}{\textit{Update choices}} \\[2pt]
LR initialization
  & Standard init ($\eta_0\approx 1$),\ small-lr init ($10^{-3}$)
  & Initial scale $\eta_0$ \\
Decay
  & None,\ scalar,\ vector
  & Weight decay \\
\bottomrule
\end{tabular}
\end{table}

\subsection{Graph memory instantiations}
\label{app:instantiated_families}

The graph memory is specified by an ordered \texttt{graph\_nodes} list and a designated \texttt{graph\_output}. Each node is one of the primitives in Table~\ref{tab:modular_ttt_all}; linear nodes carry fast weights, while activation, addition, multiplication, and normalization nodes only transform activations or gradients. The implementation executes the same graph in the train-view forward, train-view backward, and query-view forward.

\begin{figure}[ht]
\centering
\tikzset{
  fnode/.style={draw, rounded corners=3pt, align=center, font=\small,
                minimum width=1.4cm, minimum height=0.60cm, fill=gray!18},
  pnode/.style={draw, rounded corners=3pt, align=center, font=\small,
                minimum width=1.4cm, minimum height=0.60cm, fill=white},
  garr/.style={-{Stealth[length=2mm,width=1.4mm]}, semithick}
}

\begin{tikzpicture}
  \node[pnode] (inp) at (0,    0)    {input};
  \node[fnode] (l0)  at (2.2,  0)    {linear$_0$};
  \node[pnode] (a0)  at (4.4,  0)    {act};
  \node[fnode] (l1)  at (6.6,  0)    {linear$_1$};
  \node[pnode] (add) at (6.6, -1.4)  {add};
  \node[pnode] (n0)  at (8.8, -1.4)  {norm};
  \draw[garr] (inp) -- (l0);
  \draw[garr] (l0)  -- (a0);
  \draw[garr] (a0)  -- (l1);
  \draw[garr] (l1)  -- (add);
  \draw[garr] (add) -- (n0);
  \draw[garr] (inp.south) |- (add.west);
\end{tikzpicture}

\smallskip
{\small Residual graph learner}

\medskip

\begin{tikzpicture}
  \node[pnode] (inp)  at (0,    0)     {input};
  \node[fnode] (linear$_0$)   at (2.1,  0.75)  {linear$_0$};
  \node[fnode] (linear$_1$)   at (2.1, -0.75)  {linear$_1$};
  \node[pnode] (silu) at (4.0,  0.75)  {silu};
  \node[pnode] (mult) at (4.0, -0.75)  {multiply};
  \node[fnode] (linear$_2$)   at (5.9, -0.75)  {linear$_2$};
  \draw[garr] (inp.north east) -- (linear$_0$.west);
  \draw[garr] (inp.south east) -- (linear$_1$.west);
  \draw[garr] (linear$_0$)        -- (silu);
  \draw[garr] (linear$_1$)        -- (mult);
  \draw[garr] (silu.south) -- (mult.north);
  \draw[garr] (mult)      -- (linear$_2$);
\end{tikzpicture}

\smallskip
{\small SwiGLU graph learner}

\caption{Two representative graph memories. Gray nodes contain fast weights, while white nodes do not. \textbf{Top:} Residual learner; the skip connection carries input directly to add, bypassing the two linear nodes. \textbf{Bottom:} SwiGLU learner \citep{Shazeer2020GLU,Zhang2025TTTRight}; two branches share the same input and are recombined by a multiplicative gate.}
\label{fig:graph_learners}
\end{figure}
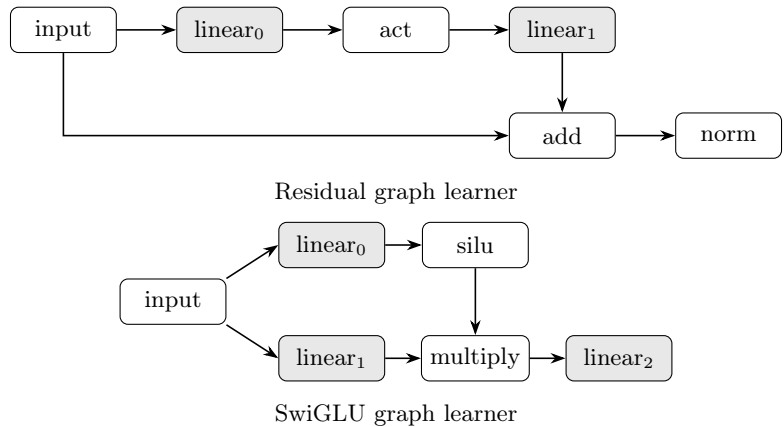

\subsection{Gradient and update rules}
\label{app:backward_rules}

Table~\ref{tab:modular_ttt_all} lists the train-view, train-view backward, and query-view rules for all primitives. This subsection records the implementation-specific details behind the train-view backward pass of Algorithm~\ref{alg:modular_ttt}: the loss gradient at the output node and the gradients used to propagate it through the graph.

\paragraph{Output gradient rules.}
All loss forms use the normalization constant $s$ from Table~\ref{tab:modular_ttt_all}. Within one chunk, the sample count in the table is $n=c$, so
\begin{equation}
s = \begin{cases} c & \text{if mean\_loss = True,} \\ 1 & \text{otherwise,} \end{cases}
\end{equation}
where $c$ is the chunk size. The corresponding output gradients are
\begin{align}
\text{MSE:}\quad
\mathbf d\hat{\mathbf V} &= \frac{\hat{\mathbf V}-\mathbf V}{s}, \\
\text{L1:}\quad
\mathbf d\hat{\mathbf V} &= \frac{\mathrm{sign}(\hat{\mathbf V}-\mathbf V)}{s}, \\
\text{Inner product:}\quad
\mathbf d\hat{\mathbf V} &= -\frac{\mathbf V}{s}, \\
\text{RMSE:}\quad
\mathbf d\hat{\mathbf V} &= \frac{\hat{\mathbf V}-\mathbf V}{\sqrt{s\,\|\hat{\mathbf V}-\mathbf V\|_F^2}}.
\end{align}
These are the same rules used in the loss-choice analysis in Appendix~\ref{app:loss_theory}. MSE and inner product preserve chunk-level scale information; L1 and RMSE respectively remove magnitude or normalize it.

\paragraph{Primitive-level backward rules.}
For $y=x\mathbf W$,
\begin{equation}
\mathbf d x = \mathbf d y\,\mathbf W^\top.
\end{equation}
The matrix is the incoming chunk state $\mathbf W_{\mathrm{start}}$, because the train-view forward and train-view backward are evaluated before the query-view write. For an elementwise activation $a$,
\begin{equation}
\mathbf d x = \mathbf d y \odot a'(x),
\end{equation}
and for root mean square normalization,
\begin{equation}
\sigma(x) = \sqrt{\frac{1}{d}\sum_{i=1}^{d} x_i^2 + \epsilon},
\qquad
y = \frac{x}{\sigma(x)},
\end{equation}
the gradient propagated to the input is
\begin{equation}
\mathbf d x
=
\frac{1}{\sigma(x)}
\left(
\widetilde{\mathbf d y} - \mathrm{mean}(y \odot \widetilde{\mathbf d y})\, y
\right).
\end{equation}
Here $\widetilde{\mathbf d y}=\mathbf w\odot\mathbf d y$ when an affine scale $\mathbf w$ (RMSNorm's $\gamma$) is present, and $\widetilde{\mathbf d y}=\mathbf d y$ otherwise. For merge primitives,
\begin{equation}
\text{if } y=x_1\odot x_2,\quad
\mathbf d x_1 = \mathbf d y \odot x_2,
\qquad
\mathbf d x_2 = \mathbf d y \odot x_1;
\qquad
\text{if } y=x_1+x_2,\quad
\mathbf d x_1=\mathbf d x_2=\mathbf d y.
\end{equation}
These local rules are the analytic backward operators used in the train-view backward without invoking nested automatic differentiation.

\subsection{Chunkwise query update}
\label{app:query_updates}

The Linear query-view rule in Table~\ref{tab:modular_ttt_all} is the chunkwise dual form of the token-wise recurrence. For a single fast linear primitive without decay,
\begin{equation}
\mathbf W_t = \mathbf W_{t-1} - \hat{\mathbf k}_t^\top \mathbf d\hat{\mathbf v}_t,
\qquad
\mathbf o_t = \mathbf q_t \mathbf W_t .
\end{equation}
Telescoping from the incoming chunk state $\mathbf W_{\mathrm{start}} = \mathbf W_0$ gives
\begin{equation}
\mathbf W_t = \mathbf W_{\mathrm{start}} - \sum_{s=1}^{t} \hat{\mathbf k}_s^\top \mathbf d\hat{\mathbf v}_s.
\end{equation}
The readout at position $t$ is therefore
\begin{equation}
\mathbf o_t = \mathbf q_t \mathbf W_t = \mathbf q_t \mathbf W_{\mathrm{start}} - \sum_{s=1}^{t} (\mathbf q_t \hat{\mathbf k}_s^\top)\, \mathbf d\hat{\mathbf v}_s.
\end{equation}
In matrix form,
\begin{equation}
\mathbf O = \mathbf Q \mathbf W_{\mathrm{start}} - \mathrm{Tril}(\mathbf Q \hat{\mathbf K}^\top)\mathbf d\hat{\mathbf V},
\qquad
\mathbf W_{\mathrm{end}} = \mathbf W_{\mathrm{start}} - \hat{\mathbf K}^\top \mathbf d\hat{\mathbf V},
\end{equation}
where $\hat{\mathbf K}$ is the scaled key in Table~\ref{tab:modular_ttt_all}, and $\mathrm{Tril}(\cdot)$ denotes the inclusive lower-triangular operator, retaining entries with $s \le t$ and zeroing entries with $s>t$ to enforce causality within the chunk. 

\subsection{Learning rate, decay, and fast-state variants}
\label{app:lr_decay_fast_states}

Learning-rate and decay factors are predicted by the outer token mixer and injected only at fast nodes. Following Table~\ref{tab:modular_ttt_all}, constant, scalar, and vector learning rates form the scaled key $\hat{\mathbf K}$:
\begin{equation}
\begin{aligned}
\hat{\mathbf K}^{(j)} &= \boldsymbol{\eta}^{(j)} \odot \mathbf K^{(j)}, && \text{for constant, scalar, and vector learning rates}.
\end{aligned}
\end{equation}

In the formulas below, $\hat{\mathbf K}^{\top}\mathbf d\hat{\mathbf V}$ denotes the effective update product after this injection. The gradient $\mathbf d\hat{\mathbf V}$ already carries the scaling factor $1/s$ from Appendix~\ref{app:backward_rules}. When mean scaling is enabled ($s=c$), this normalizes by chunk length; when $s=1$, update scale is controlled by the learned learning-rate factor and its initialization.

When vector decay is enabled, let $\phi_{t,s} = \prod_{r=s+1}^{t} f_r$ and $\phi_{t,0} = \prod_{r=1}^{t} f_r$. This is the vector-decay analogue of the decay matrix $\mathbf M$ in Table~\ref{tab:modular_ttt_all}. The chunkwise readout and end-of-chunk state update become
\begin{align}
\mathbf o_t &= \mathbf q_t \bigl(\phi_{t,0} \odot \mathbf W_{\mathrm{start}}\bigr)
        - \sum_{s=1}^{t} \mathbf q_t\bigl(\phi_{t,s} \odot \hat{\mathbf k}_s^\top\mathbf d\hat{\mathbf v}_s\bigr), \\
\mathbf W_{\mathrm{end}} &= \phi_{c,0} \odot \mathbf W_{\mathrm{start}}
                    - \sum_{s=1}^{c} \phi_{c,s} \odot \hat{\mathbf k}_s^\top \mathbf d\hat{\mathbf v}_s.
\end{align}
The fused fast-update kernel accumulates these products in log space using the supplied $\log f_t$ tensors.

We implement RMSNorm by cascading Norm and Gate. When \texttt{rmsnorm\_use\_fast\_weight} is disabled, the affine scale is {not a fast weight}. When it is enabled, the affine scale is promoted to a fast state and updated by the same causal mechanism as a linear fast node.

It also applies a persistent output gate after the graph output:
\begin{equation}
\mathbf h_{\mathrm{gate}} = \mathrm{Norm}\bigl(\mathbf h_o^{\mathrm{query}}\bigr) \odot \sigma\bigl(\mathrm{Gate}(\mathbf h)\bigr).
\end{equation}
Finally, fast linear nodes use either zero initialization (\texttt{flash}) or the Gaussian initialization of the original TTT recipe \citep{Sun2024TTT} (\texttt{official}). The initialization comparison confirms that zero initialization degrades performance across the tested configurations.

\appsection{Experimental setup}

\subsection{Model architecture}
\label{app:experiment_details}

All models use a pre-normalization residual backbone. Each block applies RMSNorm, then the token mixer, then a residual connection, followed by RMSNorm and a bias-free GLU channel mixer with SiLU activation, following the LLaMA design \citep{Touvron2023Llama2}. The channel mixer has gate, up, and down projections. For Modular TTT, the token mixer uses graph memory with query, key, value, and output projections; the large-scale Modular TTT variants use an L2-normalized key side and leave the query side unnormalized.

\begin{table}[ht]
\centering
\small
\caption{Architecture and Modular TTT hyperparameters. The scale names follow the paper notation. $d_{\mathrm{model}}$ is the hidden dimension, $L$ the number of layers, $H$ the number of heads, $d_{\mathrm{head}}$ the per-head query/key/value dimension, and $d_{\mathrm{ffn}}$ the GLU inner width.}
\label{tab:arch_hyperparams}
\begin{tabular}{lccc}
\toprule
Hyperparameter & 160M & 410M & 1.45B \\
\midrule
\multicolumn{4}{l}{\textit{Backbone}} \\
\midrule
$d_{\mathrm{model}}$ & 768 & 1024 & 2048 \\
$L$ (layers) & 12 & 24 & 24 \\
$H$ (heads) & 6 & 8 & 16 \\
$d_{\mathrm{head}}$ & 128 & 128 & 128 \\
$d_{\mathrm{ffn}}$ (GLU inner) & 2048 & 2816 & 5632 \\
Vocabulary size & 50{,}257 & 50{,}257 & 50{,}257 \\
Normalization & RMSNorm & RMSNorm & RMSNorm \\
Channel mixer & GLU & GLU & GLU \\
Channel activation & SiLU & SiLU & SiLU \\
Linear bias & no & no & no \\
Residual layout & pre-norm & pre-norm & pre-norm \\
\midrule
\multicolumn{4}{l}{\textit{Modular TTT}} \\
\midrule
Chunk size & 256 & 256 & 256 \\
Mean scaling & disabled & disabled & disabled \\
Decay & scalar & scalar & scalar \\
Inner learning rate & scalar & scalar & scalar \\
Key normalization & L2 & L2 & L2 \\
Query normalization & none & none & none \\
Fast-weight init. & Gaussian, std 0.02 & Gaussian, std 0.02 & Gaussian, std 0.02 \\
\bottomrule
\end{tabular}
\end{table}

\paragraph{Initialization.}
Unless otherwise stated, embeddings, the channel-mixer projections, and standard linear layers are initialized from a truncated normal distribution with standard deviation 0.02, with biases set to zero. In Modular TTT token mixers, the query, key, and value projections use Xavier uniform initialization with gain 0.01, while the output projection is zero-initialized by the residual rescaling rule. Fast weights use the official TTT initialization, namely Gaussian initialization with standard deviation 0.02. The scalar decay parameter is initialized so that the initial decay is approximately 0.99, and the small-lr initialization sets the initial inner learning rate to approximately $10^{-3}$.

\subsection{Training details}

All language-model runs are trained on a large-scale English pretraining corpus with the OpenAI GPT-2 BPE tokenizer (\texttt{openai-community/gpt2}; vocabulary size 50{,}257) and the next-token prediction objective. Table~\ref{tab:common_training_hyperparams} lists the shared optimization settings, and Table~\ref{tab:ablation_hyperparams} gives the scale-specific settings.

We use the standard training split of this corpus. Documents are tokenized online, concatenated into a token stream, and packed into fixed-length training sequences of length 2048 for ablation runs and 4096 for scale-up runs. We do not insert an attention mask at document boundaries, so packed sequences may cross document boundaries under the standard causal next-token objective.

\begin{table}[ht]
\centering
\small
\caption{Common training hyperparameters.}
\label{tab:common_training_hyperparams}
\begin{tabular}{ll}
\toprule
Hyperparameter & Value \\
\midrule
Optimizer & AdamW \\
Learning rate & $3\times10^{-4}$ \\
AdamW betas & $(0.9, 0.95)$ \\
AdamW $\epsilon$ & $10^{-8}$ \\
LR schedule & warmup-stable-decay \\
Warmup & 1000 steps \\
Minimum LR ratio & 0.1 \\
Precision & BF16 \\
Gradient clipping & 1.0 \\
Seed & 42 \\
\bottomrule
\end{tabular}
\end{table}

We use FSDP data-parallel sharding over the devices listed in Table~\ref{tab:ablation_hyperparams}. 

\begin{table}[ht]
\centering
\small
\caption{Training hyperparameters. Batch and eval batch are per-device batch sizes. The training token budget is computed as devices $\times$ batch $\times$ accumulation $\times$ context length $\times$ steps: the two ablation settings use $8\times32\times1\times2048\times20{,}000=10.5$B tokens, and the two scale-up settings use $16\times8\times2\times4096\times100{,}000=104.9$B tokens.}
\label{tab:ablation_hyperparams}
\begin{tabular}{lccccccc}
\toprule
Scale & Tokens & Devices & Batch & Accum. & Ctx. & Steps & Eval batch \\
\midrule
\multicolumn{8}{l}{\textit{Ablation}} \\
160M & 10B & 8 & 32 & 1 & 2048 & 20{,}000 & 192 \\
410M & 10B & 8 & 16 & 2 & 2048 & 20{,}000 & 128 \\
\midrule
\multicolumn{8}{l}{\textit{Scale-up}} \\
410M & 100B & 16 & 8 & 2 & 4096 & 100{,}000 & 128 \\
1.45B & 100B & 32 & 4 & 2 & 4096 & 100{,}000 & 64 \\
\bottomrule
\end{tabular}
\end{table}

We run downstream evaluation with \texttt{lm-eval-harness} v0.4.9.1. The evaluation covers perplexity on WikiText-103 \citep{Merity2017WikiText} and LAMBADA \citep{Paperno2016LAMBADA}, multiple-choice accuracy on BoolQ \citep{Clark2019BoolQ}, PIQA \citep{Bisk2020PIQA}, HellaSwag \citep{Zellers2019HellaSwag}, WinoGrande \citep{Sakaguchi2020WinoGrande}, ARC \citep{Clark2018ARC}, OpenBookQA \citep{Mihaylov2018OBQA}, and SIQA \citep{Sap2019SocialIQA}, and containment-style evaluation on SWDE \citep{Hao2011SWDE} and SQuAD \citep{Rajpurkar2016SQuAD}. Table~\ref{tab:eval_tasks} gives the exact harness task names and metrics.

\begin{table}[ht]
\centering
\caption{Downstream evaluation tasks. The table lists the exact \texttt{lm-eval-harness} task names and metrics used in Table~\ref{tab:large_scale_eval}. Metrics marked \texttt{acc\_norm} use length-normalized multiple-choice scoring; metrics marked \texttt{acc} use unnormalized accuracy. Group averages are arithmetic means over the listed metrics within each group.}
\label{tab:eval_tasks}
\small
\begin{tabular}{llll}
\toprule
Group & Reported as & \texttt{lm-eval-harness} task & Metric \\
\midrule
Perplexity & Wiki. & \texttt{wikitext} & \texttt{word\_perplexity} \\
Perplexity & LMB. & \texttt{lambada\_openai} & \texttt{perplexity} \\
\midrule
Multiple-choice & BoolQ & \texttt{boolq} & \texttt{acc} \\
Multiple-choice & PIQA & \texttt{piqa} & \texttt{acc} \\
Multiple-choice & Hella. & \texttt{hellaswag} & \texttt{acc\_norm} \\
Multiple-choice & Wino. & \texttt{winogrande} & \texttt{acc} \\
Multiple-choice & ARC-e & \texttt{arc\_easy} & \texttt{acc} \\
Multiple-choice & ARC-c & \texttt{arc\_challenge} & \texttt{acc\_norm} \\
Multiple-choice & OBQA & \texttt{openbookqa} & \texttt{acc\_norm} \\
Multiple-choice & SIQA & \texttt{social\_iqa} & \texttt{acc} \\
\midrule
Containment & SWDE & \texttt{swde} & \texttt{contains} \\
Containment & SQuAD & \texttt{squad\_completion} & \texttt{contains} \\
Containment & FDA & \texttt{fda} & \texttt{contains} \\
\bottomrule
\end{tabular}
\end{table}

For throughput comparisons, we report training tokens per second measured at the context length used by the corresponding run: 2048 for ablation and throughput comparisons and 4096 for scale-up runs. Throughput is averaged over 100 measured steps after a 10-step warm-up. Peak memory is the maximum reserved device memory reported by the training logger during a training step, after optimizer states have been initialized. Downstream evaluation uses the tasks and metrics listed in Table~\ref{tab:eval_tasks}.

\appsection{Analysis of Modular TTT Updates}
\label{app:theory_support}
\label{app:shallow_theory}

We analyze one chunk update under the output-gradient rules in Appendix~\ref{app:backward_rules}. The analysis tracks how each modular choice changes the scale, direction, or conditioning of the fast-weight write.

The order follows the main-text ablations. We first study update scale and small-lr init. We then compare loss functions, decay rules, shallow nonlinear modules, normalization, and multilayer fast learners.

The section stays with the operating point used in the main text; it does not introduce a separate one. When mean scaling appears below, it refers to the option $s=c$ in Table~\ref{tab:modular_ttt_all}. Equivalently, this is the \code{mean\_loss} option that divides the chunk-level output gradient by the chunk size.

\subsection{Update scale}
\label{app:update_scale}

\begin{proposition}[Chunk update scale can grow with chunk size]
\label{prop:mean_scaling}
Consider one linear fast primitive over a chunk. The train-view gradient is evaluated at the incoming fast weight $\mW_0$. The per-token update is
\begin{equation}
\Delta\mW_t = \eta\,\vk_t^\top\vd_t,
\qquad
\vd_t = \frac{\partial l}{\partial \hat{\vv}_t}/s,
\qquad
\hat{\vv}_t = \vk_t\mW_0 .
\end{equation}
Let $\Delta\mW$ denote the update subtracted from the incoming fast weight over one chunk, then $\|\Delta  W\|_F = O(\eta c / s) $.
\end{proposition}

\begin{proof}
Since
\begin{equation}
\Delta\mW = \eta/s\sum_{t=1}^{c}\vk_t^\top\vd_t,
\end{equation}
Then 
\begin{equation}
\|\Delta\mW\|_F \;\leq\; \eta/s\sum_{t=1}^{c}\|\vk_t\|_2\|\vd_t\|_2 \le \eta c / s \max\{\|\vk_t\|_2\|\vd_t\|_2 \}
= O(\eta c / s) .
\end{equation}

\end{proof}

\begin{remark}[Small-lr init as scale control]
When $s=1$, the output gradient is not divided. Instead, small-lr init sets $\eta_{\mathrm{init}}=\eta_0$ so that $c\eta_0=O(1)$. With the main-text chunk size $c=256$ and $\eta_0=10^{-3}$, the product is $c\eta_0=0.256$. Thus the initial chunk update has controlled scale even without mean scaling.

Mean scaling and small-lr init control the same initial quantity, but they differ after outer-loop training begins: mean scaling keeps the fixed factor $1/c$ in every output gradient, while small-lr init only sets the starting value of the learned learning-rate predictor.
\end{remark}

\subsection{Small-lr init (\texorpdfstring{$\eta_0 = 10^{-3}$}{eta0 = 1e-3})}
\label{app:sinlr}

The learning rate $\eta$ in Modular TTT is not a fixed scalar. It is predicted from the outer representation. Following the main text, we write it as
\begin{equation}
\eta_t = 2\,\sigma(\beta_t+b), \qquad \beta_t = \mathbf x_t^\top\mathbf w,
\end{equation}
where $\sigma$ is the sigmoid function and $b$ is a fixed scalar offset.

Without any offset ($b=0$), the initial prediction gives $\eta_t\approx 1.0$. This coefficient has the right order of magnitude for standard linear-attention-style accumulation, but it is too large for TTT updates without mean scaling. By Proposition~\ref{prop:mean_scaling}, updates without mean scaling and with $\eta\sim 1$ have $O(c)$ initial chunk scale. This agrees with the main-text observation that standard init can be unstable and is empirically worse than small-lr init.

\begin{proposition}[Small learning rates control the MSE update spectrum]
\label{prop:small_lr_spectrum}
Consider the MSE update without mean scaling used in the main-text stability discussion:
\begin{equation}
\mW^{+}
=
\mW - \rmK^\top \mathrm{diag}(\boldsymbol{\eta})(\rmK\mW-\rmV)
=
\rmA\mW + \rmK^\top \mathrm{diag}(\boldsymbol{\eta})\rmV,
\end{equation}
where
\begin{equation}
\rmA = \rmI - \rmH,
\qquad
\rmH = \rmK^\top \mathrm{diag}(\boldsymbol{\eta})\rmK.
\end{equation}
If $\lambda_{\max}(\rmH)>2$, then the homogeneous update $\mW\mapsto \rmA\mW$ amplifies some direction. If $0\le \eta_t\le \eta_0$ for all tokens, then
\begin{equation}
\lambda_{\max}(\rmH)\le \eta_0\|\rmK\|_2^2.
\end{equation}
Thus a small initial learning rate reduces the spectral scale of $\rmH$ and lowers the risk that $\rmA$ has eigenvalues with magnitude greater than one.
\end{proposition}

\begin{proof}
\textbf{Step 1: Identify the homogeneous update.}
The MSE update is affine in $\mW$. Its homogeneous part is multiplication by
\begin{equation}
\rmA = \rmI - \rmK^\top \mathrm{diag}(\boldsymbol{\eta})\rmK.
\end{equation}
This is the matrix that appears in the main-text stability discussion.

\textbf{Step 2: Relate eigenvalues of $\rmA$ and $\rmH$.}
Because $\eta_t\ge 0$, the matrix
\begin{equation}
\rmH=\rmK^\top \mathrm{diag}(\boldsymbol{\eta})\rmK
\end{equation}
is positive semidefinite. Thus all eigenvalues of $\rmH$ are nonnegative, so an eigenvalue of $\rmA=\rmI-\rmH$ can leave the unit disk only through the negative side, which occurs exactly when some $\lambda_i(\rmH)>2$. If $\lambda_i(\rmH)$ is an eigenvalue of $\rmH$, then $1-\lambda_i(\rmH)$ is the corresponding eigenvalue of $\rmA$. Therefore, if $\lambda_{\max}(\rmH)>2$, then
\begin{equation}
|1-\lambda_{\max}(\rmH)|>1,
\end{equation}
so the homogeneous update amplifies the associated eigendirection.

\textbf{Step 3: Bound the spectrum by the learning rate.}
If $0\le \eta_t\le \eta_0$, then $\mathrm{diag}(\boldsymbol{\eta})\preceq \eta_0 \rmI$. Hence
\begin{equation}
\rmH
=
\rmK^\top \mathrm{diag}(\boldsymbol{\eta})\rmK
\preceq
\eta_0 \rmK^\top\rmK,
\end{equation}
which gives
\begin{equation}
\lambda_{\max}(\rmH)
\le
\eta_0\lambda_{\max}(\rmK^\top\rmK)
=
\eta_0\|\rmK\|_2^2.
\end{equation}
Small-lr init therefore lowers the initial spectral scale of $\rmH$, making it less likely that $\rmA$ has eigenvalues with magnitude greater than one. This is the spectral counterpart of the chunk-scale control in Proposition~\ref{prop:mean_scaling}.
\end{proof}

\begin{corollary}[Mean scaling and small $\eta_0$ control the initial update]
Proposition~\ref{prop:mean_scaling} and Proposition~\ref{prop:small_lr_spectrum} describe two related effects. At the start of outer-loop training, mean scaling with $\eta=O(1)$ and updates without mean scaling but with $\eta_0=O(1/c)$ both produce an $O(1)$ chunk update scale.

For MSE, small $\eta_0$ also reduces the spectral scale of $\rmK^\top\mathrm{diag}(\boldsymbol{\eta})\rmK$. This lowers the risk that the homogeneous update matrix has eigenvalues with magnitude greater than one. After training, mean scaling continues to divide every output gradient by $c$. In contrast, small $\eta_0$ only fixes the starting value, so the predicted learning rate can grow or shrink per token as the outer representation learns. This matches the main-text setting, which uses small-lr init without mean scaling.
\end{corollary}

Small-lr init sets $b$ to a negative value chosen so that the initial learning rate is $\eta_0$. Concretely,
\begin{equation}
b = \log\!\left(\frac{p}{1-p}\right), \qquad p = \frac{\eta_0}{2},
\end{equation}
so that $2\sigma(b)=\eta_0$. With $\eta_0=10^{-3}$, this gives $p=5\times10^{-4}$ and $b\approx -7.60$.

Small-lr init has two effects. First, it gives an $O(1)$ effective per-chunk update at initialization for the same reason as mean scaling. Both reduce the initial write scale by a factor proportional to $1/c$. With $\eta_{\mathrm{init}}=10^{-3}$ and $c=256$, the scale factor is $c\cdot10^{-3}=0.256$.

Second, small $\eta_0$ does not lock the learning rate to this value. The learned projection can adjust $\eta_t$ during outer-loop training. In contrast, mean scaling is a fixed division by $c$. This distinction matches the main-text choice of small-lr init without mean scaling, while Appendix~\ref{app:mean_loss_decay_support} reports runs with mean scaling as a complementary regime.

\subsection{Loss analysis}
\label{app:loss_theory}

The loss ablation separates the four losses into two groups. MSE and inner-product loss are competitive, while L1 and RMSE are weaker. The output-gradient formulas in Appendix~\ref{app:backward_rules} show a simple mechanism behind this split: the losses differ in how much scale information their output gradients carry.

\begin{proposition}[L1 and RMSE remove residual scale]
Let $\rmR=\hat{\rmV}-\rmV$ be a nonzero residual. For any $a>0$, L1 gives the same sign pattern for $\rmR$ and $a\rmR$, while RMSE gives the same normalized gradient direction for $\rmR$ and $a\rmR$. Thus both losses remove the magnitude information carried by the residual.
\end{proposition}

\begin{proof}
For L1,
\begin{equation}
\mathrm{sign}(a\rmR)=\mathrm{sign}(\rmR)
\end{equation}
for any $a>0$ away from zero entries. The gradient depends on the sign pattern but not on the size of the error.

For RMSE,
\begin{equation}
\frac{a\rmR}{\sqrt{s\|a\rmR\|_F^2}}
=
\frac{\rmR}{\sqrt{s\|\rmR\|_F^2}},
\end{equation}
again for $a>0$ and $\rmR\ne 0$. The gradient direction is normalized, so its norm is decoupled from the reconstruction error.
\end{proof}

This mechanism matches the main-text split. MSE and inner product preserve a scale-bearing write signal, whereas L1 and RMSE collapse or normalize that scale. The latter two losses are useful controls, but they are not the default choices in the experiments.

\subsection{Multilayer fast learners}
\label{app:deep_theory}

This subsection supports the main-text deep-memory ablation. The central issue is not only model expressivity. In the one-step Modular TTT setting, depth changes the geometry of the inner update.

The zero-initialization arguments below are scope statements about a boundary case. They explain why the framework should not use all-zero fast factors for deeper product-form learners. The main experiments in Table~\ref{tab:multilayer_90m} use the Gaussian fast-weight initialization described in Appendix~\ref{app:experiment_details}, so their failures should be read as failures under nonzero initialization and one-step Modular TTT updates, not as direct consequences of the zero-init trap.

The shallow linear learner has a simple update:
\begin{equation}
\hat{\rmV}=\rmK\mW,
\qquad
\Delta\mW=\rmK^\top\rmD.
\end{equation}
Under MSE, this is the update of a convex quadratic in one fast weight. A deeper fast learner no longer has this geometry. Its factor updates depend on activations produced by upstream factors and gradients backpropagated through downstream factors.

We use this factor-coupling view throughout the subsection. In this reduced appendix version, we keep the zero-initialization trap as a clean boundary case: when all fast factors are initialized at zero, product-form depth can prevent any inner update from being written. This is not the mechanism behind the Gaussian-initialized main experiments, but it illustrates that depth changes the update geometry in ways absent from a single fast weight.

These product-form effects are closely related to classical analyses of deep linear networks and deep matrix factorization \citep{BaldiHornik1989PCA,Kawaguchi2016NoPoorLocalMinima,Arora2018ImplicitAcceleration,Arora2019DeepMatrixFactorization}.

\subsubsection{Core update identity}

Let $\rmK\in\mathbb{R}^{c\times d}$ be the chunk key matrix, $\rmV\in\mathbb{R}^{c\times e}$ be the target, and let $\rmD$ denote the effective chunk-normalized signal used by the TTT update. For a depth-$L$ product-form linear learner, the train-view output is
\begin{equation}
\hat{\rmV}=\rmK\mW_1\mW_2\cdots\mW_L .
\end{equation}
For $L=2$, we write $\mW_1\in\mathbb{R}^{d\times m}$ and $\mW_2\in\mathbb{R}^{m\times e}$.

\begin{proposition}[Product-form updates are factor-coupled]
\label{prop:factor_coupling}
For the depth-$L$ product-form learner, define
\begin{equation}
\rmA_{j-1}=\rmK\mW_1\cdots\mW_{j-1},
\qquad
\rmB_{j+1}=\mW_{j+1}\cdots\mW_L,
\end{equation}
with $\rmA_0=\rmK$ and $\rmB_{L+1}=\rmI$. The train-view update of the $j$-th factor has the form
\begin{equation}
\Delta\mW_j=\rmA_{j-1}^\top \rmD\,\rmB_{j+1}^\top .
\end{equation}
Thus the update of each factor is conditioned by the other factors on its forward-backward path.
\end{proposition}

\begin{proof}
\textbf{Step 1: Perturb one factor.}
Holding all other factors fixed, a perturbation $\delta\mW_j$ changes the output by
\begin{equation}
\delta\hat{\rmV}
=
\rmA_{j-1}\,\delta\mW_j\,\rmB_{j+1}.
\end{equation}

\textbf{Step 2: Pair with the gradient.}
Using the Frobenius inner product,
\begin{equation}
\langle \rmD,\delta\hat{\rmV}\rangle
=
\langle
\rmA_{j-1}^\top \rmD\,\rmB_{j+1}^\top,
\delta\mW_j
\rangle .
\end{equation}
Hence the update is $\Delta\mW_j=\rmA_{j-1}^\top \rmD\,\rmB_{j+1}^\top$.

\textbf{Step 3: Compare with one layer.}
When $L=1$, the formula reduces to $\Delta\mW=\rmK^\top\rmD$. For $L\ge2$, the update contains upstream activations, downstream factors, or both. This is the factor coupling absent from the shallow linear learner.
\end{proof}

\subsubsection{A boundary case: zero initialization}
\label{app:linear_linear_analysis}

The factor-coupled identity exposes a clean boundary case that separates shallow and product-form fast learners: all-zero initialization. This case is useful because it shows that depth can change whether the TTT update is written at all.

\begin{proposition}[Zero initialization gives zero gradients in product-form linear learners]
\label{prop:zero_init_trap}
Consider a product-form linear fast learner with depth $L\ge 2$. If all fast factors are initialized at zero at sequence start, then the train-view gradient of every fast factor is exactly zero. Since the update is zero, the zero state persists across chunks unless it is externally perturbed.

By contrast, a single linear fast learner with $\mW=0$ has update $\Delta\mW=\rmK^\top\rmD$, which is nonzero whenever $\rmK$ and $\rmD$ are nonzero.
\end{proposition}

\begin{proof}
\textbf{Step 1: Apply the factor-coupled update identity.}
By Proposition~\ref{prop:factor_coupling},
\begin{equation}
\Delta\mW_j=\rmA_{j-1}^\top \rmD\,\rmB_{j+1}^\top .
\end{equation}
If $L\ge 2$ and all factors are zero, then for every $j$ either $\rmA_{j-1}$ contains a zero factor or $\rmB_{j+1}$ contains a zero factor. Hence $\Delta\mW_j=0$ for all $j$.

\textbf{Step 2: Show persistence of the zero state.}
Since every factor update is zero, the fast state remains at zero after the chunk update. The same argument therefore applies again at the next chunk unless another mechanism perturbs the fast weights.

\textbf{Step 3: Compare with the single-layer case.}
For a single linear fast learner, there is no second zero factor on the update path. At $\mW=0$, the update is
\begin{equation}
\Delta\mW=\rmK^\top\rmD.
\end{equation}
This is nonzero for generic nonzero $\rmK$ and $\rmD$. Thus zero initialization does not trap a single linear factor, but traps a product-form learner with at least two zero factors.
\end{proof}

This boundary case gives one mechanism behind the main-text deep-memory ablation. Once the fast learner is parameterized as a product of multiple trainable maps, the one-step TTT update is no longer governed by the single-layer geometry. Even when Gaussian initialization avoids the zero-gradient trap used in this proposition, the update remains factor-coupled, which is the mechanism emphasized in the main text.

\appsection{Additional ablation results}
\label{app:ablation_support}

\subsection{Operating point: mean scaling, small-lr init, and decay}
\label{app:mean_loss_decay_support}

\begin{table}[ht]
\centering
\small
\setlength{\tabcolsep}{2pt}
\caption{Update-scale ablation for the linear fast learner at 2048 context. We compare standard init and small-lr init with and without mean scaling. The small-lr columns without mean scaling match the main ablation setting. Lower is better.}
\label{tab:mean_loss_decay_support}
\begin{tabular}{llcccccccc}
\toprule
\multirow{3}{*}{Loss} & \multirow{3}{*}{Decay}
 & \multicolumn{4}{c}{160M}
 & \multicolumn{4}{c}{410M} \\
\cmidrule(lr){3-6}\cmidrule(lr){7-10}
 &  & \multicolumn{2}{c}{mean scaling}
    & \multicolumn{2}{c}{no mean scaling}
    & \multicolumn{2}{c}{mean scaling}
    & \multicolumn{2}{c}{no mean scaling} \\
\cmidrule(lr){3-4}\cmidrule(lr){5-6}\cmidrule(lr){7-8}\cmidrule(lr){9-10}
 &  & standard & small-lr init
    & standard & small-lr init
    & standard & small-lr init
    & standard & small-lr init \\
\midrule
\multirow{3}{*}{MSE}
 & none   & 3.1935 & 3.2298 & 3.6012 & 3.2005 & 2.9178 & 2.9561 & 3.4036 & 2.9343 \\
 & scalar & 3.0766 & 3.0791 & 3.3035 & 3.0380 & 2.8230 & 2.8139 & 3.0820 & 2.7949 \\
 & vector & 3.0419 & 3.0506 & 3.2405 & 3.0038 & 2.7948 & 2.7975 & 3.0548 & 2.7821 \\
\multirow{3}{*}{Inner}
 & none   & 3.2111 & 3.2314 & 3.2407 & 3.2028 & 2.9362 & 2.9601 & 2.9572 & 2.9273 \\
 & scalar & 3.0730 & 3.0774 & 3.0538 & 3.0383 & 2.8253 & 2.8164 & 2.8054 & 2.7938 \\
 & vector & 3.0402 & 3.0474 & 3.0033 & 3.0048 & 2.7983 & 2.7986 & 2.7793 & 2.7822 \\
\bottomrule
\end{tabular}
\end{table}

Table~\ref{tab:mean_loss_decay_support} separates two ways of controlling the inner-update scale. Mean scaling fixes the gradient scale by dividing by the chunk length, while small-lr init controls the initial learned learning rate and then lets the outer model adapt it. The standard-init rows without mean scaling are consistently worse, especially under MSE, which is the instability regime discussed in the main text. Once the scale is controlled, the decay ordering is stable: no decay is weakest, scalar decay gives most of the gain, and vector decay gives the lowest loss when it is available. This supports the main-text operating point: small-lr init without mean scaling, with scalar decay used as the default efficiency--quality trade-off.

\subsection{Shallow operators}
\label{app:linear_mean_family}

\begin{table}[ht]
\centering
\small
\setlength{\tabcolsep}{3pt}
\caption{Single-layer operator ablation at 2048 context. The columns without mean scaling match the main nonlinear-operator sweep, while the mean-scaling columns give standard-init controls. All rows use scalar decay; Norm denotes RMSNorm.
 Lower is better.}
\label{tab:linear_mean_family}
\begin{tabular}{llcccc}
\toprule
\multirow{3}{*}{Size} & \multirow{3}{*}{Method}
  & \multicolumn{2}{c}{MSE}
  & \multicolumn{2}{c}{Inner-product} \\
\cmidrule(lr){3-4}\cmidrule(lr){5-6}
  & & Mean scaling & No mean scaling & Mean scaling & No mean scaling \\
  & & standard & small-lr init & standard & small-lr init \\
\midrule
\multirow{4}{*}{160M}
 & Linear & 3.0766 & 3.0380 & 3.0730 & 3.0383 \\
 & Linear + GELU & 3.0872 & 3.0232 & 3.0859 & 3.0225 \\
 & Linear + SiLU & 3.0867 & \textbf{3.0205} & 3.0744 & 3.0242 \\
 & Linear + Norm & 3.0929 & 3.0300 & 3.0814 & 3.0335 \\
\midrule
\multirow{4}{*}{410M}
 & Linear & 2.8230 & 2.7949 & 2.8253 & 2.7938 \\
 & Linear + GELU & 2.8155 & 2.7844 & 2.8150 & \textbf{2.7832} \\
 & Linear + SiLU & 2.8139 & 2.7849 & 2.8148 & 2.7838 \\
 & Linear + Norm & 2.8511 & 2.8192 & 2.9009 & 2.7899 \\
\bottomrule
\end{tabular}
\end{table}

Table~\ref{tab:linear_mean_family} shows why the main text treats the activation gain as a shallow-operator effect tied to the chosen scale regime. With mean scaling and standard init, GELU and SiLU stay close to the Linear row. Without mean scaling and with small-lr init, both activations consistently improve over Linear at 160M and 410M. Norm is more sensitive to loss and scale, so it is not used as the default shallow extension.

\paragraph{Robustness.}
We compare MSE with inner-product loss for the Linear learner and compare Linear with Linear-SiLU under MSE using paired seeds $\{1,24,42,48,84\}$. All other training and evaluation settings are fixed. Table~\ref{tab:five_seed_robustness} reports the final validation loss at step 20{,}000 for each run, together with the five-seed mean and sample standard deviation.

\begin{table*}[t]
\centering
\small
\setlength{\tabcolsep}{4pt}
\caption{Multi-run results using seeds $\{1,24,42,48,84\}$. Entries are final validation losses at step 20{,}000; the last column reports the mean $\pm$ sample standard deviation. Lower is better.}
\label{tab:five_seed_robustness}
\begin{tabular}{llcccccc}
\toprule
Scale & Setting & Seed 1 & Seed 24 & Seed 42 & Seed 48 & Seed 84 & Mean $\pm$ std \\
\midrule
160M & Linear, MSE                & 3.0431 & 3.0371 & 3.0380 & 3.0384 & 3.0339 & $3.0381 \pm 0.0033$ \\
160M & Linear, inner-product loss & 3.0413 & 3.0374 & 3.0383 & 3.0386 & 3.0338 & $3.0379 \pm 0.0027$ \\
160M & Linear-SiLU, MSE           & 3.0202 & 3.0246 & 3.0205 & 3.0219 & 3.0215 & \textbf{$3.0217 \pm 0.0017$} \\
\midrule
410M & Linear, MSE                & 2.7981 & 2.7965 & 2.7949 & 2.7954 & 2.7919 & $2.7954 \pm 0.0023$ \\
410M & Linear, inner-product loss & 2.7923 & 2.7940 & 2.7938 & 2.7951 & 2.7917 & $2.7934 \pm 0.0014$ \\
410M & Linear-SiLU, MSE           & 2.7839 & 2.7953 & 2.7849 & 2.7953 & 2.7906 & \textbf{$2.7900 \pm 0.0055$} \\
\bottomrule
\end{tabular}
\end{table*}

The difference between MSE and inner-product loss is small and does not establish a consistent ordering across scales. Linear-SiLU outperforms its paired Linear-MSE run for all five seeds at both scales, supporting a consistent gain from the SiLU activation.

\subsection{Normalization diagnostics}
\label{app:normalization_support}

\paragraph{Normalization and $\epsilon$.}

\label{app:passive_eps}

\begin{table}[ht]
\centering
\small
\setlength{\tabcolsep}{4pt}
\caption{Normalization and $\epsilon$ at 160M for \code{Linear + Norm}. Lower is better.}
\label{tab:passive_norm_eps}
\begin{tabular}{lccc}
\toprule
 $\epsilon$ & Mean scaling + MSE & No mean scaling + MSE & No mean scaling + Inner \\
\midrule
$10^{-5}$ & 3.0929 & 3.0735 & 3.0787 \\
$10^{-3}$ & 3.0744 & 3.0631 & 3.0447 \\
$10^{-2}$ & 3.0733 & 3.0300 & 3.0335 \\
$0.05$    & 3.0467 & 3.0341 & 3.0392 \\
$0.1$     & 3.0475 & 3.0381 & 3.0387 \\
$0.5$     & 3.0967 & 3.0371 & 3.0377 \\
$1.0$     & 3.0884 & 3.0412 & 3.0419 \\
\bottomrule
\end{tabular}
\end{table}

Table~\ref{tab:passive_norm_eps} confirms that normalization is strongly $\epsilon$-dependent. Increasing $\epsilon$ can reduce this gradient amplification, but this is a stabilization effect rather than a robust quality improvement~{\ref{tab:stabilized_multilayer}, \ref{tab:linear_norm_linear_norm_eps}, and \ref{tab:residual_memories}}. The best Norm rows without mean scaling remain behind the GELU and SiLU rows in Table~\ref{tab:linear_mean_family}, so the main text attributes the shallow gain to simple pointwise activations rather than normalization.

\paragraph{SiLU followed by normalization.}
\label{app:silu_norm_shallow}

\begin{table}[ht]
\centering
\small
\setlength{\tabcolsep}{5pt}
\caption{\code{Linear + SiLU + Norm} at 2048 context. Rows use $\epsilon=10^{-2}$, scalar decay, and small-lr init, without mean scaling. Lower is better.}
\label{tab:silu_norm_shallow}
\begin{tabular}{lcc}
\toprule
 & 160M & 410M \\
\midrule
MSE   & 3.1543 & $\times$ \\
Inner & $\times$ & $\times$ \\
\bottomrule
\end{tabular}
\end{table}

Table~\ref{tab:silu_norm_shallow} tests whether Norm composes with the SiLU gain. In this setting it does not: the only completed row is worse than both \code{Linear + SiLU} and \code{Linear + Norm}, while the other rows either diverge or stay in failed high-loss regimes. This supports keeping the shallow recommendation simple: one linear fast map followed by a pointwise activation.

\paragraph{Fast-weight normalization.}
\label{app:fw_norm_support}

\begin{table}[ht]
\centering
\small
\setlength{\tabcolsep}{4pt}
\caption{Fast-weight normalization at 160M. We vary mean scaling, small-lr init, and $\epsilon$ for the same topology. Lower is better.}
\label{tab:fw_norm}
\begin{tabular}{ccccc}
\toprule
Mean scaling & Small-lr init & $\epsilon$ & MSE & Inner \\
\midrule
\multirow{3}{*}{no}
 & \multirow{3}{*}{yes}
 & default & $\times$ & $\times$ \\
 &  & 0.25    & $\times$ & $\times$ \\
 &  & 1.0     & $\times$ & $\times$ \\
\midrule
\multirow{3}{*}{yes}
 & \multirow{3}{*}{no}
 & default & 3.2719 & 3.2724 \\
 &  & 0.25    & 3.1942 & 3.2007 \\
 &  & 1.0     & 3.1989 & 3.2060 \\
\midrule
\multirow{3}{*}{yes}
 & \multirow{3}{*}{yes}
 & default & 3.1835 & 3.2046 \\
 &  & 0.25    & 3.1831 & \textbf{3.1683} \\
 &  & 1.0     & 3.1737 & 3.1695 \\
\bottomrule
\end{tabular}
\end{table}

Table~\ref{tab:fw_norm} isolates the stability factors for fast-weight normalization. Without mean scaling, the tested rows fail; with mean scaling, the runs become trainable, but they remain worse than the strongest shallow activation rows in Table~\ref{tab:linear_mean_family}. These runs therefore serve as diagnostics for the normalization mechanism rather than candidates for the final model family. 

\subsection{Multilayer and gated memories}
\label{app:family_support}

This subsection reports the deeper and gated-memory sweeps behind Table~\ref{tab:multilayer_90m}. The goal is to test whether a more expressive inner learner improves the one-step Modular TTT setting after the update scale has been stabilized. All rows use the 160M scale, 2048-token context, MSE loss, scalar decay, small-lr init, and Gaussian fast-weight initialization, unless a table states otherwise. The shallow reference is \code{Linear-SiLU}, with validation loss 3.0205.

We organize the evidence in three steps. First, Table~\ref{tab:multilayer_family_frontier} lists the best completed row from each deeper family. Second, Tables~\ref{tab:product_activation_depth} and \ref{tab:stabilized_multilayer} give same-setting sweeps over depth, activation placement, mean scaling, and normalization. Third, Tables~\ref{tab:residual_memories} and \ref{tab:fast_weight_init_ablation} report stability and initialization checks for residual and gated memories.

\begin{table}[H]
\centering
\small
\setlength{\tabcolsep}{5pt}
\caption{Best completed rows for deeper and gated memories at 160M. $\Delta$ is the validation-loss gap to the shallow \code{Linear-SiLU} reference, 3.0205. Lower is better.}
\label{tab:multilayer_family_frontier}
\begin{tabular*}{\textwidth}{@{\extracolsep{\fill}}p{0.24\textwidth}p{0.40\textwidth}cc@{}}
\toprule
Family & Topology & 160M & $\Delta$ \\
\midrule
Shallow reference & Linear-SiLU & 3.0205 & -- \\
\midrule
Product-form linear & Linear-Linear & 3.1265 & +0.1060 \\
Activation-only & Linear-SiLU-Linear-SiLU & 3.1156 & +0.0951 \\
Norm-containing & Linear-SiLU-Linear-Norm ($\epsilon=0.5$) & \textbf{3.1144} & +0.0939 \\
Residual & $f_{\mathrm{res\text{-}norm}}$ ($\epsilon=0.5$) & 3.1643 & +0.1438 \\
Gated & SwiGLU + Norm & 3.2041 & +0.1836 \\
\bottomrule
\end{tabular*}
\end{table}

\paragraph{Best completed rows.}
Table~\ref{tab:multilayer_family_frontier} gives the main result. Some deeper memories can be made trainable, but none improves over the shallow \code{Linear-SiLU} reference. The best deeper row is \code{Linear-SiLU-Linear-Norm} at $\epsilon=0.5$, and it is still worse by 0.0939 validation loss.

\begin{table}[H]
\centering
\small
\setlength{\tabcolsep}{6pt}
\caption{Two-layer linear and activation-only memories at 160M. Rows use MSE loss, scalar decay, small-lr init, and Gaussian fast-weight initialization. The grid varies SiLU placement and mean scaling. Lower is better.}
\label{tab:product_activation_depth}
\begin{tabular*}{\textwidth}{@{\extracolsep{\fill}}p{0.45\textwidth}cc@{}}
\toprule
Topology & Mean scaling & No mean scaling \\
\midrule
Linear-Linear             & 3.1888 & 3.1265 \\
Linear-Linear-SiLU        & 3.1876 & 3.1223 \\
Linear-SiLU-Linear        & 3.1845 & 3.1240 \\
Linear-SiLU-Linear-SiLU   & 3.1767 & \textbf{3.1156} \\
\bottomrule
\end{tabular*}
\end{table}

\paragraph{Matched depth sweeps.}
Table~\ref{tab:product_activation_depth} varies where SiLU is placed in a two-linear-layer memory. Rows without mean scaling are better than rows with mean scaling in this grid. Adding SiLU between the two linear maps or after the second map helps slightly, but the gain is small compared with the gap to \code{Linear-SiLU}. All eight cells remain behind the shallow reference. This matches the product-form analysis in Appendix~\ref{app:deep_theory}: adding fast factors changes the inner-update geometry, and a pointwise activation does not remove this coupling.

\begin{table}[H]
\centering
\small
\setlength{\tabcolsep}{6pt}
\caption{Norm-containing two-layer memories at 160M. Rows use MSE loss, scalar decay, small-lr init, and Gaussian fast-weight initialization. The default is $\epsilon=10^{-5}$. Lower is better.}
\label{tab:stabilized_multilayer}
\begin{tabular*}{\textwidth}{@{\extracolsep{\fill}}p{0.45\textwidth}cc@{}}
\toprule
Topology & Mean scaling & No mean scaling \\
\midrule
Linear-Linear-Norm        & 3.2409   & $\times$ \\
Linear-Norm-Linear        & 3.1858   & \textbf{3.1594} \\
Linear-SiLU-Linear-Norm   & 3.2792   & $\times$ \\
Linear-SiLU-Linear-Norm, $\epsilon=0.5$ & 3.1144 & 3.1939 \\
Linear-Norm-Linear-Norm   & $\times$ & $\times$ \\
\bottomrule
\end{tabular*}
\end{table}

\begin{table}[H]
\centering
\small
\setlength{\tabcolsep}{6pt}
\caption{Mean-scaling epsilon sweep for Linear-Norm-Linear-Norm at 160M. Rows use MSE loss, scalar decay, small-lr init, and Gaussian fast-weight initialization. Lower is better; $\times$ denotes divergence.}
\label{tab:linear_norm_linear_norm_eps}
\begin{tabular*}{\textwidth}{@{\extracolsep{\fill}}p{0.45\textwidth}c@{}}
\toprule
Setting & Mean scaling \\
\midrule
Linear-Norm-Linear-Norm, $\epsilon=0.1$  & $\times$ \\
Linear-Norm-Linear-Norm, $\epsilon=0.25$ & 3.4962 \\
Linear-Norm-Linear-Norm, $\epsilon=0.5$  & 3.4755 \\
Linear-Norm-Linear-Norm, $\epsilon=0.75$ & 3.1750 \\
Linear-Norm-Linear-Norm, $\epsilon=1.0$  & 3.1637 \\
\bottomrule
\end{tabular*}
\end{table}

Table~\ref{tab:stabilized_multilayer} gives the matched comparison for norm-containing depth. All trainable rows remain behind the shallow reference. Increasing $\epsilon$ above the default $10^{-5}$ stabilizes some rows without mean scaling, but it does not close the gap: \code{Linear-SiLU-Linear-Norm} reaches 3.1144 at $\epsilon=0.5$. For the deeper \textbf{\code{Linear-Norm-Linear-Norm}} case, the mean-scaling sweep in Table~\ref{tab:linear_norm_linear_norm_eps} reaches 3.1637 at $\epsilon=1.0$.

\paragraph{Chunk-size robustness.}
We sweep $c\in\{128,256,512\}$ at 160M and 410M, keeping all training parameters fixed and varying only $c$ within each learner graph. Table~\ref{tab:chunk_size_robustness} reports the final validation loss at step 20{,}000.

\begin{table}[H]
\centering
\small
\setlength{\tabcolsep}{5pt}
\caption{Chunk-size robustness at 160M and 410M. Within each learner graph, all training parameters are fixed and only the chunk size $c$ changes. Entries are final validation losses at step 20{,}000; lower is better.}
\label{tab:chunk_size_robustness}
\begin{tabular*}{\textwidth}{@{\extracolsep{\fill}}lcccccc@{}}
\toprule
\multirow{2}{*}{Learner graph}
& \multicolumn{3}{c}{160M} & \multicolumn{3}{c}{410M} \\
\cmidrule(lr){2-4}\cmidrule(lr){5-7}
& $c=128$ & $c=256$ & $c=512$ & $c=128$ & $c=256$ & $c=512$ \\
\midrule
Linear-GELU & 3.0224 & 3.0232 & 3.0247 & 2.7834 & \textbf{2.7844} & \textbf{2.7838} \\
Linear-SiLU & \textbf{3.0210} & \textbf{3.0205} & \textbf{3.0217} & \textbf{2.7833} & 2.7849 & 2.7845 \\
Linear-Norm & 3.0375 & 3.0300 & 3.0285 & 2.7961 & 2.8192 & 2.8066 \\
Linear-SiLU-Linear-Norm & 3.1373 & 3.1144 & 3.1132 & 2.9295 & 2.9258 & 2.9128 \\
\bottomrule
\end{tabular*}
\end{table}

Linear-GELU and Linear-SiLU form the strongest group in all six settings. Linear-Norm remains behind these activation-only shallow learners, and the deeper \code{Linear-SiLU-Linear-Norm} learner remains behind every single-layer learner across all evaluated chunk sizes and model scales.

For residual memories, we test two graph forms. Let $\sigma$ denote SiLU. The first applies an activation after the residual addition,
\begin{equation}
f_{\mathrm{res\text{-}act}}(\vk)
=
\sigma\!\left(\vk + \sigma(\vk\mW_0)\mW_1\right),
\end{equation}
and the second applies output RMSNorm after the same residual addition,
\begin{equation}
f_{\mathrm{res\text{-}norm}}(\vk)
=
\mathrm{RMSNorm}_{\epsilon}\!\left(\vk + \sigma(\vk\mW_0)\mW_1\right).
\end{equation}
These expressions correspond directly to the graph \code{input -> linear -> SiLU -> linear -> add(input, linear) -> final operator}.

\begin{table}[H]
\centering
\small
\setlength{\tabcolsep}{6pt}
\caption{Residual and gated memories at 160M. Rows use MSE loss, scalar decay, small-lr init, and Gaussian fast-weight initialization. The residual rows should be compared with their non-residual counterparts: $f_{\mathrm{res\text{-}act}}$ with Linear-SiLU-Linear-SiLU, and $f_{\mathrm{res\text{-}norm}}$ with Linear-SiLU-Linear-Norm at matched $\epsilon$. Lower is better.}
\label{tab:residual_memories}
\begin{tabular*}{\textwidth}{@{\extracolsep{\fill}}llcc@{}}
\toprule
Family & Setting & Mean scaling & No mean scaling \\
\midrule
\multirow{3}{*}{Baseline for Residual}
& Linear-SiLU-Linear-SiLU & 3.1767          & 3.1156 \\
& Linear-SiLU-Linear-Norm, $\epsilon=10^{-5}$ & 3.2792          & $\times$ \\
& Linear-SiLU-Linear-Norm, $\epsilon=0.5$ & 3.1144          & 3.1939 \\
\midrule
\multirow{3}{*}{Residual}
 & $f_{\mathrm{res\text{-}act}}$ & 3.2367          & $\times$ \\
 & $f_{\mathrm{res\text{-}norm}}$, $\epsilon=10^{-5}$ & $\times$ & $\times$ \\
 & $f_{\mathrm{res\text{-}norm}}$, $\epsilon=0.5$ & 3.2068          & \textbf{3.1643} \\
\midrule
SwiGLU
 & $-$                             & $\times$        & $\times$ \\
\midrule
\multirow{4}{*}{SwiGLU + Norm}
 & $\epsilon=10^{-5}$              & $\times$        & $\times$ \\
 & $\epsilon=0.25$                 & 3.2097          & $\times$ \\
 & $\epsilon=0.5$                  & 3.2126          & $\times$ \\
 & $\epsilon=1.0$                  & \textbf{3.2041} & $\times$ \\
\bottomrule
\end{tabular*}
\end{table}

\paragraph{Residual and gated checks.}
According to Table~\ref{tab:multilayer_90m}, unlike the chain-structured rows kept in the main text, these memories introduce additive or multiplicative branch interactions, so their stability depends more strongly on mean scaling and RMSNorm $\epsilon$. 
Table~\ref{tab:residual_memories} shows that these extensions do not improve the controlled one-step TTT setting. The residual rows remain worse than their non-residual counterparts in Tables~\ref{tab:product_activation_depth} and \ref{tab:stabilized_multilayer}, while SwiGLU requires additional stabilization and still stays below the shallow reference. Thus, the residual and gated checks support the same conclusion as the main deep-memory table: adding branch structure increases stabilization sensitivity but does not close the gap to the shallow Linear-SiLU frontier.

\begin{table}[H]
\centering
\small
\setlength{\tabcolsep}{6pt}
\caption{Fast-weight initialization ablation at 160M with MSE loss. Zero init sets all fast weights to zero; Gaussian init follows the official TTT initialization. Lower is better.}
\label{tab:fast_weight_init_ablation}
\begin{tabular*}{\textwidth}{@{\extracolsep{\fill}}p{0.45\textwidth}cc@{}}
\toprule
Topology & Zero init & Gaussian init \\
\midrule
\multicolumn{3}{l}{\textit{No mean scaling}} \\
\quad Single Linear              & 3.0343   & 3.0380 \\
\quad Linear-SiLU                & 3.0247   & \textbf{3.0205} \\
\quad Linear-Linear              & $\times$ & 3.1265 \\
\quad Linear-SiLU-Linear-SiLU    & $\times$ & 3.1156 \\
\quad SwiGLU + Norm              & $\times$ & $\times$ \\
\midrule
\multicolumn{3}{l}{\textit{Mean scaling}} \\
\quad SwiGLU + Norm              & $\times$ & 3.2041 \\
\bottomrule
\end{tabular*}
\end{table}

Table~\ref{tab:fast_weight_init_ablation} connects the initialization study to Appendix~\ref{app:deep_theory}. Zero initialization works for the single linear and shallow \code{Linear-SiLU} rows, but it fails for the product-form deeper and gated rows. Gaussian initialization avoids this zero-gradient failure mode. Even with Gaussian initialization, however, deeper and gated memories still do not beat the shallow \code{Linear-SiLU} reference.

Together, these sweeps support the main-text conclusion: in the tested one-step Modular TTT setting, increasing the fast-learner depth or adding residual and gated structure does not improve over the shallow \code{Linear-SiLU} memory.

\subsection{Baselines and related model comparisons}
\label{app:external_repro}

This appendix gathers the external-system checks used to interpret the main comparisons. These checks should not be read as isolated graph-memory ablations. The baseline implementations package several choices together: the official TTT code has separate handwritten \texttt{linear} and \texttt{mlp} learner paths, while the LaCT code combines sliding-window softmax attention, a SwiGLU-style fast learner, optional TTT branching, normalization, learning-rate prediction, and fused kernels inside one layer. This is precisely why the main paper uses Modular TTT for controlled ablations: in our implementation, the graph topology, loss, learning-rate rule, decay rule, and update schedule are exposed as separate design dimensions rather than being tied to a model-specific update path.

\subsubsection{Official TTT comparison}
\label{app:official_ttt_supp}

\begin{table}[ht]
\centering
\small
\caption{Official TTT reproduction baselines at 160M and 2048 context. Lower is better.}
\label{tab:official_ttt_compare}
\begin{tabular}{lcc}
\toprule
Configuration & Source & 160M \\
\midrule
Linear, no norm & Official TTT & 3.3953 \\
TTT-Linear (Linear + norm) & Official TTT & 3.3951 \\
TTT-MLP (Linear-GELU-linear + norm) & Official TTT & 3.4061 \\
\bottomrule
\end{tabular}
\end{table}

Table~\ref{tab:official_ttt_compare} records the completed reproduction runs using the official TTT implementation \citep{Sun2024TTT}. The official implementation is organized around two fixed learner classes. The linear path maintains a single fast matrix and bias, uses a layer-normalized L2 reconstruction target, and implements separate dual-form and cache-aware update logic. The MLP path repeats the same structure with two fast matrices, a GELU nonlinearity, explicit backward computations for both factors, and topology-specific state dictionaries. As a result, changing the learner family also changes the update derivation, normalization placement, and cache state layout.

This differs from Modular TTT. In our implementation, a graph is built from primitive nodes such as Linear, Act, Norm, Add, and Multiply; the same three-pass procedure runs train-view forward, train-view backward, and query-view forward for any valid graph. Loss choice and decay are selected independently of the graph. The official runs in Table~\ref{tab:official_ttt_compare} are therefore useful as reproduction baselines, but they are not component-matched ablations against a single Modular TTT setting.

The relatively high losses in Table~\ref{tab:official_ttt_compare} are consistent with this recipe mismatch. The official path uses a fixed reconstruction-style TTT objective with residual target \(\mathbf V-\mathbf K\), a learned learning-rate gate scaled by the head dimension, no scalar/vector decay in these reproduction configs, and topology-specific normalization/update code. In the main controlled ablations, each of these dimensions matters: no decay is substantially worse than scalar decay, normalization is mixed, and deeper product-form learners are fragile. Thus the official rows combine several weaker choices under our training setup, rather than isolating a single disadvantage. The gap should therefore be read as evidence that a modular implementation can match or improve quality while exposing the relevant design knobs, not as a claim that the official code is a controlled lower bound for every TTT recipe.

\subsubsection{Inference efficiency}
\label{app:inference_efficiency}

We evaluated inference efficiency using the 160M and 410M model configurations. The model architectures and training configurations match the corresponding main experiments, and all models follow the same inference measurement protocol. Table~\ref{tab:inference_efficiency} reports prefill throughput, decode throughput, and peak allocated device memory.

During autoregressive decoding, Modular TTT updates the fast-weight state only at chunk boundaries. Tokens in the current incomplete chunk still use all updates accumulated up to the current position, so the associated computation grows within the chunk and introduces additional decode overhead relative to GDN's token-wise recurrent update. We leave optimization of this decode path to future work.

\begin{table}[H]
\centering
\small
\setlength{\tabcolsep}{4pt}
\caption{Inference at 160M and 410M. \code{M} denotes Modular TTT-Linear, \code{M-SiLU} denotes Modular TTT-Linear-SiLU, and Inner denotes inner-product loss. Higher throughput and lower peak allocated memory are better.}
\label{tab:inference_efficiency}
\begin{tabular*}{\textwidth}{@{\extracolsep{\fill}}llrrr@{}}
\toprule
Scale & Model & Prefill (tokens/s) & Decode (tokens/s) & Peak allocated (GiB) \\
\midrule
\multirow{6}{*}{160M}
 & GDN & 700{,}983.9 & 1{,}821.36 & 1.634 \\
 & LaCT & 264{,}403.2 & 113.69 & 2.154 \\
 & M (MSE) & 449{,}824.0 & 790.45 & 1.500 \\
 & M-SiLU (MSE) & 404{,}772.1 & 700.55 & 1.500 \\
 & M (Inner) & 448{,}892.3 & 789.93 & 1.500 \\
 & M-SiLU (Inner) & 401{,}684.6 & 685.03 & 1.500 \\
\midrule
\multirow{6}{*}{410M}
 & GDN & 227{,}305.4 & 400.70 & 1.779 \\
 & LaCT & 86{,}888.0 & 37.46 & 2.043 \\
 & M (MSE) & 130{,}244.3 & 219.67 & 1.751 \\
 & M-SiLU (MSE) & 116{,}602.9 & 192.31 & 1.689 \\
 & M (Inner) & 128{,}931.6 & 219.13 & 1.689 \\
 & M-SiLU (Inner) & 112{,}546.9 & 201.49 & 1.751 \\
\bottomrule
\end{tabular*}
\end{table}

GDN provides the highest prefill and decode throughput at both scales. Every evaluated Modular TTT configuration is faster than the current LaCT implementation and uses less peak allocated memory than GDN and LaCT. Within Modular TTT, the Linear variants provide the highest throughput, and the choice between MSE and inner-product loss has little effect on inference efficiency.

\subsubsection{LaCT comparison}
\label{app:lact_check}

\begin{table}[ht]
\centering
\small
\caption{External systems check in a LaCT-style setting with sliding-window size 512 and chunk size 512. The table reports final validation loss. Lower is better. Absolute loss is comparable only within tokenizer-matched groups.}
\label{tab:lact_compare}
\begin{tabular}{lcc}
\toprule
Configuration & 160M & 930M (official in paper) \\
\midrule
LaCT SwiGLU, official tokenizer, no TTT & 2.7769 & 2.4012 \\
LaCT SwiGLU, official tokenizer, TTT path & 2.7978 & 2.4009 \\
\bottomrule
\end{tabular}
\end{table}

We also repeated the comparison in the LaCT setting \citep{Zhang2025TTTRight}, using sliding-window size 512 together with chunk size 512. The LaCT implementation first computes sliding-window softmax attention and then optionally adds a TTT branch. When the branch is enabled, the code reuses or separately projects Q/K/V for the fast learner, applies SiLU and L2 normalization to the fast views, predicts per-branch learning rates, and dispatches to either a linear or SwiGLU update path. The SwiGLU path uses three fast weights and the form
\[
f(x)=W_1\bigl(\mathrm{SiLU}(W_0x)\odot W_2x\bigr),
\]
with optional momentum, Muon-style update normalization, fast-weight norm control, and fused Triton kernels. Thus LaCT changes the fast learner, local attention path, schedule, optimizer details, and normalization together.

Table~\ref{tab:lact_compare} should therefore be interpreted as a boundary-regime systems check. With the official tokenizer and this strong sliding-window setting, disabling the TTT path is comparable to enabling it: the no-TTT variant is better at 160M, while the 930M official-scale rows are essentially tied. This does not contradict the main Modular TTT ablations. Rather, it shows that in this LaCT configuration the sliding-window softmax path carries much of the performance, so the result is not an isolated test of the SwiGLU graph learner. The controlled Modular TTT sweeps in the main text and Appendix~\ref{app:family_support} are the relevant evidence for the claim that, under our one-step TTT operating point, deeper and gated graph learners do not improve over the shallow frontier.

\subsection{Per-token loss diagnostics}
\label{app:per_token_loss}
Figure~\ref{fig:per_token_loss} examines how loss evolves along the sequence. The selected Modular TTT variants remain stable around and beyond the 4k training context, without the sharp post-context degradation observed in some baselines. This suggests that the shallow variants selected by the ablation are not only competitive in aggregate loss, but also behave robustly across token positions.

\begin{figure}[!t]
\centering
\includegraphics[width=0.9\linewidth]{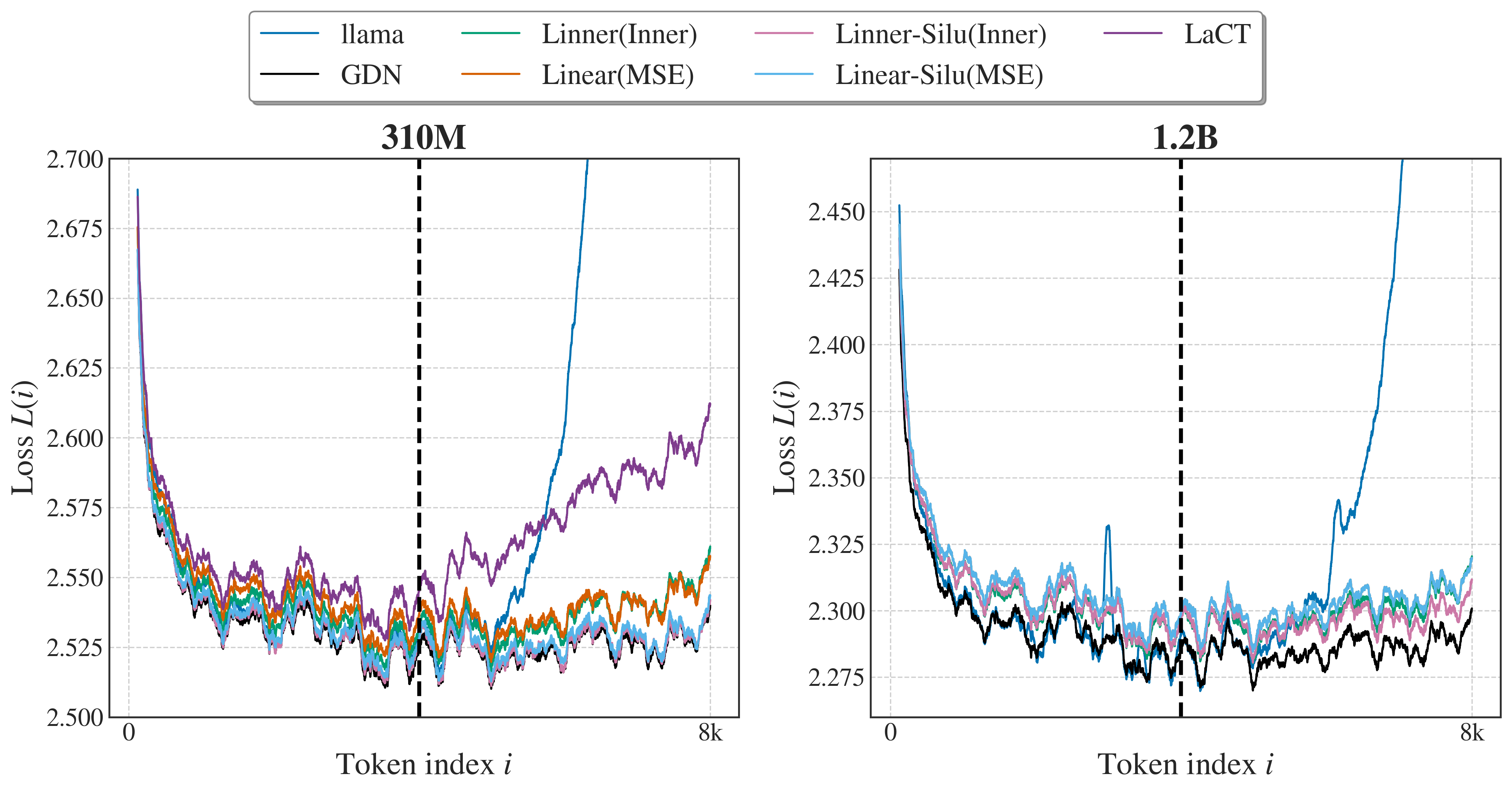}
\caption{Per-token loss along an 8k sequence for representative 410M shallow Modular TTT variants. The dashed line marks the 4k training context used in the large-scale setting. Lower is better.}
\label{fig:per_token_loss}
\end{figure}

\subsection{RULER evaluation details}
\label{app:ruler_details}

We evaluate explicit retrieval with the needle-in-a-haystack (NIAH) tasks from RULER \citep{Hsieh2024RULER}. The main text reports the \texttt{niah\_single} family at 1k--8k as a compact view of direct single-key retrieval. This appendix gives the complete NIAH results used for that analysis, including \texttt{niah\_multikey}, \texttt{niah\_single}, \texttt{niah\_multiquery}, and \texttt{niah\_multivalue} at 1k--8k. For \texttt{niah\_multikey} and \texttt{niah\_single}, the family-level rows average the corresponding three templates; \texttt{niah\_multiquery} and \texttt{niah\_multivalue} are single task families.

The full results confirm the main-text conclusion. The selected Modular TTT variants are competitive with recurrent baselines on some short- and mid-context NIAH settings, but they do not match LLaMA on precise retrieval. The gap is largest at 4k--8k and on harder multi-key templates, where several models are near zero. Across Modular TTT variants, inner-product loss is often stronger for explicit retrieval at 410M, while the SiLU variants are less consistently beneficial than in language-modeling loss. These results suggest that the current shallow fast-weight memories improve efficient sequence modeling but are not yet reliable exact-recall mechanisms.

\begin{table*}[!t]
\centering
\small
\setlength{\tabcolsep}{3pt}
\renewcommand{\arraystretch}{0.96}
\caption{Family-level \texttt{niah\_single} RULER results at 410M and 1.45B. Entries are accuracies in percent. Higher is better.}
\label{tab:ruler_single_family_full}
\begin{tabular*}{\textwidth}{@{\extracolsep{\fill}}llrrrr@{}}
\toprule
Scale & Model & 1k & 2k & 4k & 8k \\
\midrule
\multirow{7}{*}{410M}
 & LLaMA & 97.93 & 93.33 & 80.47 & 51.87 \\
 & GDN & 87.73 & 53.27 & 38.40 & 15.33 \\
 & LaCT & 84.67 & 64.20 & 33.80 & 7.13 \\
 & M (MSE) & 76.27 & 62.27 & 26.40 & 8.00 \\
 & M-SiLU (MSE) & 74.33 & 50.67 & 25.27 & 13.60 \\
 & M (Inner) & 81.67 & 55.73 & 28.93 & 14.07 \\
 & M-SiLU (Inner) & 61.27 & 42.33 & 15.60 & 8.47 \\
\midrule
\multirow{6}{*}{1.45B}
 & LLaMA & 99.00 & 94.13 & 99.20 & 92.13 \\
 & GDN & 86.60 & 67.00 & 42.20 & 18.07 \\
 & M (MSE) & 75.13 & 67.40 & 48.67 & 21.07 \\
 & M-SiLU (MSE) & 88.60 & 64.93 & 52.87 & 15.00 \\
 & M (Inner) & 88.93 & 79.87 & 50.00 & 18.20 \\
 & M-SiLU (Inner) & 79.73 & 78.07 & 46.20 & 18.73 \\
\bottomrule
\end{tabular*}
\end{table*}

\begin{table*}[!t]
\centering
\small
\setlength{\tabcolsep}{3pt}
\renewcommand{\arraystretch}{0.96}
\caption{Family-level RULER results for the three multi-NIAH families at 410M and 1.45B. Entries are accuracies in percent. Higher is better.}
\label{tab:ruler_multi_family_full}
\begin{tabular*}{\textwidth}{@{\extracolsep{\fill}}lllrrrr@{}}
\toprule
Scale & Family & Model & 1k & 2k & 4k & 8k \\
\midrule
\multirow{21}{*}{410M}
 & \multirow{7}{*}{\texttt{niah\_multikey}} & LLaMA & 26.07 & 18.53 & 17.07 & 9.93 \\
 &  & GDN & 7.80 & 7.67 & 7.27 & 3.73 \\
 &  & LaCT & 14.13 & 9.53 & 6.60 & 3.73 \\
 &  & M (MSE) & 8.33 & 7.20 & 6.93 & 3.33 \\
 &  & M-SiLU (MSE) & 8.33 & 7.40 & 5.87 & 5.20 \\
 &  & M (Inner) & 8.00 & 7.40 & 7.33 & 4.93 \\
 &  & M-SiLU (Inner) & 9.53 & 7.53 & 7.20 & 4.67 \\
\cmidrule(lr){2-7}
 & \multirow{7}{*}{\texttt{niah\_multiquery}} & LLaMA & 66.10 & 46.55 & 26.80 & 4.05 \\
 &  & GDN & 22.35 & 20.25 & 14.10 & 6.05 \\
 &  & LaCT & 25.35 & 20.95 & 10.80 & 3.50 \\
 &  & M (MSE) & 24.25 & 21.55 & 17.25 & 6.40 \\
 &  & M-SiLU (MSE) & 23.55 & 21.10 & 14.90 & 3.00 \\
 &  & M (Inner) & 45.15 & 26.80 & 20.40 & 11.45 \\
 &  & M-SiLU (Inner) & 28.45 & 23.85 & 18.75 & 7.65 \\
\cmidrule(lr){2-7}
 & \multirow{7}{*}{\texttt{niah\_multivalue}} & LLaMA & 61.55 & 46.45 & 28.45 & 9.15 \\
 &  & GDN & 30.60 & 20.00 & 17.15 & 9.00 \\
 &  & LaCT & 29.70 & 21.60 & 12.05 & 6.40 \\
 &  & M (MSE) & 33.55 & 24.90 & 16.70 & 5.30 \\
 &  & M-SiLU (MSE) & 31.40 & 22.15 & 15.70 & 4.80 \\
 &  & M (Inner) & 44.50 & 28.65 & 22.95 & 11.25 \\
 &  & M-SiLU (Inner) & 33.75 & 22.40 & 18.65 & 11.95 \\
\midrule
\multirow{18}{*}{1.45B}
 & \multirow{6}{*}{\texttt{niah\_multikey}} & LLaMA & 43.47 & 34.80 & 30.80 & 20.13 \\
 &  & GDN & 14.40 & 10.60 & 8.73 & 3.53 \\
 &  & M (MSE) & 16.60 & 12.07 & 8.60 & 4.87 \\
 &  & M-SiLU (MSE) & 12.53 & 9.67 & 7.40 & 3.07 \\
 &  & M (Inner) & 15.20 & 12.33 & 8.20 & 3.93 \\
 &  & M-SiLU (Inner) & 12.07 & 9.73 & 5.87 & 3.80 \\
\cmidrule(lr){2-7}
 & \multirow{6}{*}{\texttt{niah\_multiquery}} & LLaMA & 89.20 & 85.25 & 65.40 & 21.95 \\
 &  & GDN & 70.35 & 48.90 & 19.35 & 4.70 \\
 &  & M (MSE) & 64.75 & 54.10 & 28.65 & 3.10 \\
 &  & M-SiLU (MSE) & 50.20 & 30.20 & 17.15 & 4.75 \\
 &  & M (Inner) & 73.50 & 49.05 & 20.75 & 8.00 \\
 &  & M-SiLU (Inner) & 44.85 & 26.55 & 17.35 & 2.80 \\
\cmidrule(lr){2-7}
 & \multirow{6}{*}{\texttt{niah\_multivalue}} & LLaMA & 85.10 & 72.70 & 52.20 & 22.35 \\
 &  & GDN & 73.60 & 56.00 & 24.30 & 12.35 \\
 &  & M (MSE) & 68.90 & 52.80 & 35.05 & 10.15 \\
 &  & M-SiLU (MSE) & 51.85 & 30.05 & 18.10 & 5.90 \\
 &  & M (Inner) & 76.60 & 52.15 & 20.05 & 6.95 \\
 &  & M-SiLU (Inner) & 50.45 & 22.90 & 16.70 & 3.15 \\
\bottomrule
\end{tabular*}
\end{table*}

\appsection{Limitations}
\label{app:limitations}

\paragraph{Limitations.}
Our study is limited to autoregressive language modeling under the training budgets and model scales reported in the experiments. Although Modular TTT exposes a broad set of graph, loss, learning-rate, decay, and normalization choices, it does not exhaust all possible inner learners or optimization schedules; in particular, deeper and gated learners may behave differently under different update rules, chunk sizes, optimizers, or hybrid attention designs. This scope also excludes system-specific update refinements used in LaCT, such as Muon-style update normalization and momentum variants, because they do not fit the current efficient fused graph implementation without adding specialized update paths; in our LaCT-style checks, these refinements did not produce a clear quality gain. Empirically, the selected shallow Modular TTT variants remain weaker than LLaMA on containment-style tasks and explicit long-context retrieval, indicating that fixed-state TTT still has limitations for precise recall. These limitations suggest that future work should study richer update schedules, better retrieval-oriented memory mechanisms, and broader evaluations beyond the language-modeling setting considered here.

%% file: main.bib
@inproceedings{Vaswani+2017,
 author = {Vaswani, Ashish and Shazeer, Noam and Parmar, Niki and Uszkoreit, Jakob and Jones, Llion and Gomez, Aidan N. and Kaiser, Lukasz and Polosukhin, Illia},
 booktitle = {Advances in Neural Information Processing Systems},
 title = {Attention Is All You Need},
 volume = {30},
 year = {2017}
}

@inproceedings{Katharopoulos2020LinearAttention,
 author = {Katharopoulos, Angelos and Vyas, Apoorv and Pappas, Nikolaos and Fleuret, Fran{\c{c}}ois},
 booktitle = {Proceedings of the 37th International Conference on Machine Learning},
 title = {Transformers Are RNNs: Fast Autoregressive Transformers with Linear Attention},
 year = {2020}
}

@inproceedings{GuDao2024Mamba,
 author = {Gu, Albert and Dao, Tri},
 booktitle = {First Conference on Language Modeling},
 title = {{Mamba}: Linear-Time Sequence Modeling with Selective State Spaces},
 year = {2024}
}

@inproceedings{DaoGu2024SSM,
 author = {Dao, Tri and Gu, Albert},
 title = {Transformers Are {SSM}s: Generalized Models and Efficient Algorithms Through Structured State Space Duality},
 booktitle = {Proceedings of the 41st International Conference on Machine Learning},
 series = {Proceedings of Machine Learning Research},
 volume = {235},
 pages = {10041--10071},
 year = {2024},
 publisher = {PMLR}
}

@article{Sun2024TTT,
 author = {Sun, Yu and Li, Xinhao and Dalal, Karan and Xu, Jiarui and Vikram, Arjun and Zhang, Genghan and Dubois, Yann and Chen, Xinlei and Wang, Xiaolong and Koyejo, Sanmi and Hashimoto, Tatsunori and Guestrin, Carlos},
 journal = {arXiv preprint arXiv:2407.04620},
 title = {Learning to (Learn at Test Time): {RNN}s with Expressive Hidden States},
 year = {2024}
}

@inproceedings{Yang2025DeltaNet,
 author = {Yang, Songlin and Wang, Bailin and Zhang, Yu and Shen, Yikang and Kim, Yoon},
 booktitle = {Advances in Neural Information Processing Systems},
 volume = {37},
 title = {Parallelizing Linear Transformers with the Delta Rule over Sequence Length},
 year = {2024},
 doi = {10.52202/079017-3668},
 url = {https://proceedings.neurips.cc/paper_files/paper/2024/hash/d13a3eae72366e61dfdc7eea82eeb685-Abstract-Conference.html}
}

@inproceedings{Yang2025GDN,
 author = {Yang, Songlin and Kautz, Jan and Hatamizadeh, Ali},
 booktitle = {International Conference on Learning Representations},
 title = {Gated Delta Networks: Improving {Mamba}2 with Delta Rule},
 year = {2025}
}

@article{Behrouz2025Titans,
 author = {Behrouz, Ali and Zhong, Peilin and Mirrokni, Vahab},
 journal = {arXiv preprint arXiv:2501.00663},
 title = {{Titans}: Learning to Memorize at Test Time},
 year = {2025}
}

@article{Tandon2025E2ETTT,
 author = {Tandon, Arnuv and Dalal, Karan and Li, Xinhao and Koceja, Daniel and R{\o}d, Marcel and Buchanan, Sam and Wang, Xiaolong and Leskovec, Jure and Koyejo, Sanmi and Hashimoto, Tatsunori and Guestrin, Carlos and McCaleb, Jed and Choi, Yejin and Sun, Yu},
 journal = {arXiv preprint arXiv:2512.23675},
 title = {End-to-End Test-Time Training for Long Context},
 year = {2025}
}

@article{Zhang2025TTTRight,
 author = {Zhang, Tianyuan and Bi, Sai and Hong, Yicong and Zhang, Kai and Luan, Fujun and Yang, Songlin and Sunkavalli, Kalyan and Freeman, William T. and Tan, Hao},
 journal = {arXiv preprint arXiv:2505.23884},
 title = {Test-Time Training Done Right},
 year = {2025}
}

@article{Li2025TNT,
 author = {Li, Zeman and Behrouz, Ali and Deng, Yuan and Zhong, Peilin and Kacham, Praneeth and Karami, Mahdi and Razaviyayn, Meisam and Mirrokni, Vahab},
 journal = {arXiv preprint arXiv:2511.07343},
 title = {{TNT}: Improving Chunkwise Training for Test-Time Memorization},
 year = {2025}
}

@article{Han2025ViT3,
 author = {Han, Dongchen and Li, Yining and Li, Tianyu and Cao, Zixuan and Wang, Ziming and Song, Jun and Cheng, Yu and Zheng, Bo and Huang, Gao},
 journal = {arXiv preprint arXiv:2512.01643},
 title = {ViT$^3$: Unlocking Test-Time Training in Vision},
 year = {2025}
}

@article{Behrouz2025Atlas,
 author = {Behrouz, Ali and Li, Zeman and Kacham, Praneeth and Daliri, Majid and Deng, Yuan and Zhong, Peilin and Razaviyayn, Meisam and Mirrokni, Vahab},
 journal = {arXiv preprint arXiv:2505.23735},
 title = {{ATLAS}: Learning to Optimally Memorize the Context at Test Time},
 year = {2025}
}

@article{Behrouz2025Miras,
 author = {Behrouz, Ali and Razaviyayn, Meisam and Zhong, Peilin and Mirrokni, Vahab},
 journal = {arXiv preprint arXiv:2504.13173},
 title = {It's All Connected: A Journey Through Test-Time Memorization, Attentional Bias, Retention, and Online Optimization},
 year = {2025}
}

@article{Behrouz2025Nested,
 author = {Behrouz, Ali and Razaviyayn, Meisam and Zhong, Peilin and Mirrokni, Vahab},
 journal = {arXiv preprint arXiv:2512.24695},
 title = {Nested Learning: The Illusion of Deep Learning Architectures},
 year = {2025}
}

@article{Schmidhuber1992FastWeights,
 author = {Schmidhuber, J{\"u}rgen},
 journal = {Neural Computation},
 number = {1},
 pages = {131--139},
 title = {Learning to Control Fast-Weight Memories: An Alternative to Dynamic Recurrent Networks},
 volume = {4},
 year = {1992}
}

@inproceedings{Ba2016FastWeights,
 author = {Ba, Jimmy and Hinton, Geoffrey E. and Mnih, Volodymyr and Leibo, Joel Z. and Ionescu, Catalin},
 booktitle = {Advances in Neural Information Processing Systems},
 title = {Using Fast Weights to Attend to the Recent Past},
 volume = {29},
 year = {2016}
}

@inproceedings{Schlag2021FWP,
 author = {Schlag, Imanol and Irie, Kazuki and Schmidhuber, J{\"u}rgen},
 booktitle = {Proceedings of the 38th International Conference on Machine Learning},
 title = {Linear Transformers Are Secretly Fast Weight Programmers},
 year = {2021}
}

@article{Touvron2023Llama2,
 author = {Touvron, Hugo and Martin, Louis and Stone, Kevin and Albert, Peter and Almahairi, Amjad and Babaei, Yasmine and Bashlykov, Nikolay and Batra, Soumya and Bhargava, Prajjwal and Bhosale, Shruti and others},
 journal = {arXiv preprint arXiv:2307.09288},
 title = {Llama 2: Open Foundation and Fine-Tuned Chat Models},
 year = {2023}
}

@inproceedings{Gu2022S4,
 author = {Gu, Albert and Goel, Karan and R{\'e}, Christopher},
 booktitle = {International Conference on Learning Representations},
 title = {Efficiently Modeling Long Sequences with Structured State Spaces},
 year = {2022}
}

@article{Sun2023RetNet,
 author = {Sun, Yutao and Dong, Li and Huang, Shaohan and Ma, Shuming and Xia, Yuqing and Xue, Jilong and Wang, Jianyong and Wei, Furu},
 journal = {arXiv preprint arXiv:2307.08621},
 title = {Retentive Network: A Successor to Transformer for Large Language Models},
 year = {2023}
}

@inproceedings{Yang2024GLA,
 author = {Yang, Songlin and Wang, Bailin and Shen, Yikang and Panda, Rameswar and Kim, Yoon},
 booktitle = {International Conference on Machine Learning},
 title = {Gated Linear Attention Transformers with Hardware-Efficient Training},
 year = {2024}
}

@inproceedings{Qin2024HGRN2,
 author = {Qin, Zhen and Yang, Songlin and Sun, Weixuan and Shen, Xuyang and Li, Dong and Sun, Weigao and Zhong, Yiran},
 booktitle = {First Conference on Language Modeling},
 title = {{HGRN2}: Gated Linear {RNN}s with State Expansion},
 year = {2024}
}

@inproceedings{Qin2023HGRN,
 author = {Qin, Zhen and Yang, Songlin and Zhong, Yiran},
 booktitle = {Advances in Neural Information Processing Systems},
 title = {Hierarchically Gated Recurrent Neural Network for Sequence Modeling},
 year = {2023}
}

@inproceedings{Choromanski2021Performer,
 author = {Choromanski, Krzysztof and Likhosherstov, Valerii and Dohan, David and Song, Xingyou and Gane, Andreea and Sarlos, Tamas and Hawkins, Peter and Davis, Jared and Mohiuddin, Afroz and Kaiser, Lukasz and Belanger, David and Colwell, Lucy and Weller, Adrian},
 booktitle = {International Conference on Learning Representations},
 title = {Rethinking Attention with Performers},
 year = {2021}
}

@inproceedings{Zellers2019HellaSwag,
 author = {Zellers, Rowan and Holtzman, Ari and Bisk, Yonatan and Farhadi, Ali and Choi, Yejin},
 booktitle = {Proceedings of the 57th Annual Meeting of the Association for Computational Linguistics},
 title = {HellaSwag: Can a Machine Really Finish Your Sentence?},
 year = {2019}
}

@inproceedings{Bisk2020PIQA,
 author = {Bisk, Yonatan and Zellers, Rowan and Le Bras, Ronan and Gao, Jianfeng and Choi, Yejin},
 booktitle = {Proceedings of the AAAI Conference on Artificial Intelligence},
 title = {{PIQA}: Reasoning about Physical Commonsense in Natural Language},
 year = {2020}
}

@article{Clark2018ARC,
 author = {Clark, Peter and Cowhey, Isaac and Etzioni, Oren and Khot, Tushar and Sabharwal, Ashish and Schoenick, Carissa and Tafjord, Oyvind},
 journal = {arXiv preprint arXiv:1803.05457},
 title = {Think You Have Solved Question Answering? Try ARC, the AI2 Reasoning Challenge},
 year = {2018}
}

@inproceedings{Sakaguchi2020WinoGrande,
 author = {Sakaguchi, Keisuke and Le Bras, Ronan and Bhagavatula, Chandra and Choi, Yejin},
 booktitle = {Proceedings of the AAAI Conference on Artificial Intelligence},
 title = {{WinoGrande}: An Adversarial Winograd Schema Challenge at Scale},
 year = {2020}
}

@misc{eval-harness,
 author = {Sutawika, Lintang and Schoelkopf, Hailey and Gao, Leo and Abbasi, Baber and Biderman, Stella and Tow, Jonathan and fattori, ben and Lovering, Charles and farzanehnakhaee70 and Phang, Jason and Thite, Anish and Fazz and Aflah and Niklas and Wang, Thomas and sdtblck and nopperl and gakada and tttyuntian and researcher2 and Etxaniz, Julen and Chris and Lee, Hanwool Albert and Sinev, Leonid and Kasner, Zden{\v e}k and Stokes, Kiersten and Khalid and KonradSzafer and Hsu, Jeffrey and Kanekar, Anjor},
 title = {{EleutherAI/lm-evaluation-harness}: v0.4.9.1},
 publisher = {Zenodo},
 doi = {10.5281/zenodo.16737642},
 url = {https://zenodo.org/records/16737642},
 month = aug,
 year = {2025}
}

@inproceedings{Hsieh2024RULER,
 author = {Hsieh, Cheng-Ping and Sun, Simeng and Kriman, Samuel and Acharya, Shantanu and Rekesh, Dima and Jia, Fei and Ginsburg, Boris},
 title = {{RULER}: What's the Real Context Size of Your Long-Context Language Models?},
 booktitle = {First Conference on Language Modeling},
 year = {2024},
 url = {https://openreview.net/forum?id=kIoBbc76Sy}
}

@misc{yang2025flame,
 title  = {Flame: Flash Language Modeling Made Easy},
 author = {Zhang, Yu and Yang, Songlin},
 url    = {https://github.com/fla-org/flame},
 month  = jan,
 year   = {2025}
}

@article{Shazeer2020GLU,
 author = {Shazeer, Noam},
 title = {{GLU} Variants Improve Transformer},
 journal = {arXiv preprint arXiv:2002.05202},
 year = {2020}
}

@inproceedings{Krause2018Dynamic,
 author = {Krause, Ben and Kahembwe, Emmanuel and Murray, Iain and Renals, Steve},
 title = {Dynamic Evaluation of Neural Sequence Models},
 booktitle = {Proceedings of the 35th International Conference on Machine Learning},
 year = {2018}
}

@inproceedings{Irie2021RecurrentFWP,
 author = {Irie, Kazuki and Schlag, Imanol and Csord\'{a}s, R\'{o}bert and Schmidhuber, J\"{u}rgen},
 title = {Going Beyond Linear Transformers with Recurrent Fast Weight Programmers},
 booktitle = {Advances in Neural Information Processing Systems},
 volume = {34},
 year = {2021}
}

@inproceedings{Irie2022Dual,
 author = {Irie, Kazuki and Csord\'{a}s, R\'{o}bert and Schmidhuber, J\"{u}rgen},
 title = {The Dual Form of Neural Networks Revisited: Connecting Test Time Predictions to Training Patterns via Spotlights of Attention},
 booktitle = {Proceedings of the 39th International Conference on Machine Learning},
 year = {2022}
}

@inproceedings{Clark2019BoolQ,
 author = {Clark, Christopher and Lee, Kenton and Chang, Ming-Wei and Kwiatkowski, Tom and Collins, Michael and Toutanova, Kristina},
 title = {{BoolQ}: Exploring the Surprising Difficulty of Natural Yes/No Questions},
 booktitle = {Proceedings of the 2019 Conference of the North American Chapter of the Association for Computational Linguistics: Human Language Technologies},
 pages = {2622--2632},
 year = {2019}
}

@inproceedings{Mihaylov2018OBQA,
 author = {Mihaylov, Todor and Clark, Peter and Khot, Tushar and Sabharwal, Ashish},
 title = {Can a Suit of Armor Conduct Electricity? {A} New Dataset for Open Book Question Answering},
 booktitle = {Proceedings of the 2018 Conference on Empirical Methods in Natural Language Processing},
 pages = {2618--2628},
 year = {2018}
}

@inproceedings{Sap2019SocialIQA,
 author = {Sap, Maarten and Rashkin, Hannah and Chen, Derek and Le Bras, Ronan and Choi, Yejin},
 title = {Social {IQ}a: Commonsense Reasoning about Social Interactions},
 booktitle = {Proceedings of the 2019 Conference on Empirical Methods in Natural Language Processing and the 9th International Joint Conference on Natural Language Processing},
 pages = {4463--4473},
 year = {2019},
 doi = {10.18653/v1/D19-1454}
}

@inproceedings{Rajpurkar2016SQuAD,
 author = {Rajpurkar, Pranav and Zhang, Jian and Lopyrev, Konstantin and Liang, Percy},
 title = {{SQuAD}: 100{,}000+ Questions for Machine Comprehension of Text},
 booktitle = {Proceedings of the 2016 Conference on Empirical Methods in Natural Language Processing},
 pages = {2383--2392},
 year = {2016}
}

@inproceedings{Merity2017WikiText,
 author = {Merity, Stephen and Xiong, Caiming and Bradbury, James and Socher, Richard},
 title = {Pointer Sentinel Mixture Models},
 booktitle = {International Conference on Learning Representations},
 year = {2017}
}

@inproceedings{Paperno2016LAMBADA,
 author = {Paperno, Denis and Kruszewski, Germ\'{a}n and Lazaridou, Angeliki and Pham, Quan Ngoc and Bernardi, Raffaella and Pezzelle, Sandro and Baroni, Marco and Boleda, Gemma and Fern\'{a}ndez, Raquel},
 title = {The {LAMBADA} Dataset: Word Prediction Requiring a Broad Discourse Context},
 booktitle = {Proceedings of the 54th Annual Meeting of the Association for Computational Linguistics},
 pages = {1525--1534},
 year = {2016}
}

@inproceedings{Hao2011SWDE,
 author = {Hao, Qiang and Cai, Rui and Pang, Yanwei and Zhang, Lei},
 title = {From One Tree to a Forest: A Unified Solution for Structured Web Data Extraction},
 booktitle = {Proceedings of the 34th International ACM SIGIR Conference on Research and Development in Information Retrieval},
 pages = {775--784},
 year = {2011}
}

@inproceedings{Martin2018ParallelLRNN,
 author = {Martin, Eric and Cundy, Chris},
 title = {Parallelizing Linear Recurrent Neural Nets Over Sequence Length},
 booktitle = {International Conference on Learning Representations},
 year = {2018}
}

@inproceedings{Orvieto2023LRU,
 author = {Orvieto, Antonio and Smith, Samuel L. and Gu, Albert and Fernando, Anushan and Gulcehre, Caglar and Pascanu, Razvan and De, Soham},
 title = {Resurrecting Recurrent Neural Networks for Long Sequences},
 booktitle = {Proceedings of the 40th International Conference on Machine Learning},
 year = {2023}
}

@article{Hochreiter1997LSTM,
 author = {Hochreiter, Sepp and Schmidhuber, J{\"u}rgen},
 title = {Long Short-Term Memory},
 journal = {Neural Computation},
 volume = {9},
 number = {8},
 pages = {1735--1780},
 year = {1997}
}

@article{Chung2014GRU,
 author = {Chung, Junyoung and Gulcehre, Caglar and Cho, KyungHyun and Bengio, Yoshua},
 title = {Empirical Evaluation of Gated Recurrent Neural Networks on Sequence Modeling},
 journal = {arXiv preprint arXiv:1412.3555},
 year = {2014}
}

@inproceedings{Peng2023RWKV,
 author = {Peng, Bo and Alcaide, Eric and Anthony, Quentin and Albalak, Alon and Arcadinho, Samuel and Biderman, Stella and Cao, Huanqi and Cheng, Xin and Chung, Michael and Derczynski, Leon and Du, Xingjian and Grella, Matteo and Gv, Kranthi and He, Xuzheng and Hou, Haowen and Kazienko, Przemyslaw and Kocon, Jan and Kong, Jiaming and Koptyra, Bart{\l}omiej and Lau, Hayden and Lin, Jiaju and Mantri, Krishna Sri Ipsit and Mom, Ferdinand and Saito, Atsushi and Song, Guangyu and Tang, Xiangru and Wind, Johan and Wo{\'z}niak, Stanis{\l}aw and Zhang, Zhenyuan and Zhou, Qinghua and Zhu, Jian and Zhu, Rui-Jie},
 title = {{RWKV}: Reinventing {RNN}s for the Transformer Era},
 booktitle = {Findings of the Association for Computational Linguistics: EMNLP 2023},
 pages = {14048--14077},
 year = {2023},
 address = {Singapore},
 publisher = {Association for Computational Linguistics},
 doi = {10.18653/v1/2023.findings-emnlp.936},
 url = {https://aclanthology.org/2023.findings-emnlp.936/}
}

@article{Peng2025RWKV7,
 author = {Peng, Bo and Zhang, Ruichong and Goldstein, Daniel and Alcaide, Eric and Du, Xingjian and Hou, Haowen and Lin, Jiaju and Liu, Jiaxing and Lu, Janna and Merrill, William and Song, Guangyu and Tan, Kaifeng and Utpala, Saiteja and Wilce, Nathan and Wind, Johan S. and Wu, Tianyi and Wuttke, Daniel and Zhou-Zheng, Christian},
 title = {{RWKV-7} ``Goose'' with Expressive Dynamic State Evolution},
 journal = {arXiv preprint arXiv:2503.14456},
 year = {2025},
 doi = {10.48550/arXiv.2503.14456},
 url = {https://arxiv.org/abs/2503.14456}
}

@inproceedings{Kawaguchi2016NoPoorLocalMinima,
 author = {Kawaguchi, Kenji},
 title = {Deep Learning without Poor Local Minima},
 booktitle = {Advances in Neural Information Processing Systems},
 volume = {29},
 pages = {586--594},
 year = {2016}
}

@article{BaldiHornik1989PCA,
 author = {Baldi, Pierre and Hornik, Kurt},
 title = {Neural Networks and Principal Component Analysis: Learning from Examples without Local Minima},
 journal = {Neural Networks},
 volume = {2},
 number = {1},
 pages = {53--58},
 year = {1989},
 doi = {10.1016/0893-6080(89)90014-2}
}

@inproceedings{Arora2018ImplicitAcceleration,
 author = {Arora, Sanjeev and Cohen, Nadav and Hazan, Elad},
 title = {On the Optimization of Deep Networks: Implicit Acceleration by Overparameterization},
 booktitle = {Proceedings of the 35th International Conference on Machine Learning},
 volume = {80},
 series = {Proceedings of Machine Learning Research},
 pages = {244--253},
 year = {2018}
}

@inproceedings{Arora2019DeepMatrixFactorization,
 author = {Arora, Sanjeev and Cohen, Nadav and Hu, Wei and Luo, Yuping},
 title = {Implicit Regularization in Deep Matrix Factorization},
 booktitle = {Advances in Neural Information Processing Systems},
 volume = {32},
 pages = {7411--7422},
 year = {2019}
}

@misc{qin2024transnormerllm,
      title={TransNormerLLM: A Faster and Better Large Language Model with Improved TransNormer},
      author={Zhen Qin and Dong Li and Weigao Sun and Weixuan Sun and Xuyang Shen and Xiaodong Han and Yunshen Wei and Baohong Lv and Xiao Luo and Yu Qiao and Yiran Zhong},
      year={2024},
      eprint={2307.14995},
      archivePrefix={arXiv},
      primaryClass={cs.CL}
}

@inproceedings{
grazzi2024unlocking,
title={Unlocking State-Tracking in Linear {RNN}s Through Negative Eigenvalues},
author={Riccardo Grazzi and Julien Siems and J{\"o}rg K.H. Franke and Arber Zela and Frank Hutter and Massimiliano Pontil},
booktitle={NeurIPS 2024 Workshop on Mathematics of Modern Machine Learning},
year={2024},
url={https://openreview.net/forum?id=EXGAodFkQX}
}
